\documentclass[journal]{IEEEtai}

\usepackage[colorlinks,urlcolor=blue,linkcolor=blue,citecolor=blue]{hyperref}

\usepackage{color,array}

\usepackage{ifpdf}
\usepackage{amsthm}
\usepackage{amsmath} 
\newtheorem{theorem}{Theorem}
\theoremstyle{definition}
\newtheorem{definition}{Definition}
\newtheorem{remark}{Remark}

\usepackage{xcolor}
\usepackage{hyperref}
\usepackage{graphics} 
\usepackage{epsfig} 
\usepackage{times} 
\usepackage{amsmath} 
\usepackage{amssymb}  
\usepackage{bm}
\usepackage{cite}

\usepackage{amscd}

\usepackage{graphicx}
\usepackage[ruled,linesnumbered]{algorithm2e}
\usepackage{algpseudocode}
\usepackage{graphics}
\usepackage{booktabs}

\DeclareMathOperator*{\argmax}{argmax}
\DeclareMathOperator*{\argmin}{argmin}

\usepackage[alpha]{mdpn}  

\usepackage{cite}

\usepackage{graphics} 
\usepackage{epsfig} 
\usepackage{subfigure}
\ifCLASSINFOpdf
\else
  \usepackage[dvips]{graphicx}
\fi
\usepackage{amsmath}
\usepackage{bm}
\usepackage{amssymb}

\usepackage{array}

\usepackage{stfloats}
\ifCLASSOPTIONcaptionsoff
  \usepackage[nomarkers]{endfloat}
 \let\MYoriglatexcaption\caption
    \renewcommand{\caption}[2][\relax]{\MYoriglatexcaption[#2]{#2}}
\fi
\usepackage{url}

\usepackage[pagewise,switch]{lineno}

\usepackage{fancyhdr}

\fancypagestyle{plain}{
  \fancyhf{}
  \fancyhead[C]{Conference on \LaTeX}     
  \fancyfoot[L]{This is a notice}

}

\begin{document}
%
\title{Learning Multi-Pursuit Evasion for Safe Targeted Navigation of Drones}
%
%
%
    
\author{Jiaping Xiao,~\IEEEmembership{Graduate Student Member,~IEEE}~and~Mir Feroskhan,~\IEEEmembership{Member,~IEEE}
\thanks{This work was supported by A*STAR under its 2021 MTC Individual Research Grants (IRG) and Young Individual Research Grants (YIRG) Call (Award MTCYIRG21-0004i). \textit{(Corresponding author: Mir Feroskhan.)}}

\thanks{J. Xiao and M. Feroskhan are with the School of Mechanical and Aerospace Engineering, Nanyang Technological University, Singapore 639798, Singapore (e-mail: jiaping001@e.ntu.edu.sg; mir.feroskhan@ntu.edu.sg).
}}

\maketitle

\begin{abstract}
Safe navigation of drones in the presence of adversarial physical attacks from multiple pursuers is a challenging task. This paper proposes a novel approach, asynchronous multi-stage deep reinforcement learning (AMS-DRL), to train adversarial neural networks that can learn from the actions of multiple evolved pursuers and adapt quickly to their behavior, enabling the drone to avoid attacks and reach its target. Specifically, AMS-DRL evolves adversarial agents in a pursuit-evasion game where the pursuers and the evader are asynchronously trained in a bipartite graph way during multiple stages. Our approach guarantees convergence by ensuring Nash equilibrium among agents from the game-theory analysis. We evaluate our method in extensive simulations and show that it outperforms baselines with higher navigation success rates. We also analyze how parameters such as the relative maximum speed affect navigation performance. Furthermore, we have conducted physical experiments and validated the effectiveness of the trained policies in real-time flights. A success rate heatmap is introduced to elucidate how spatial geometry influences navigation outcomes. Project website: \url{https://github.com/NTU-ICG/AMS-DRL-for-Pursuit-Evasion}.

\end{abstract}

\begin{IEEEImpStatement}
Safe autonomous navigation is fundamentally required for intelligent drones, especially in congested air traffic scenarios. Obstacle avoidance approaches such as traditional path planning and emerging reinforcement learning-based controllers have been widely studied and deployed. However, most of the existing methods assume static motion or unchanged behavior of obstacles, falling short of addressing the intelligent attack from intentional agents like birds. Inspired by the pursuit-evasion behavior in nature, the reinforcement learning-based solution proposed in this paper provides drones with rapid response and evolving capabilities to evade intelligent pursuers. With the asynchronous multi-stage training framework, this solution is supposed to bring insights into navigation methods on many other platforms, including self-driving cars, ground mobile robots, and underwater vehicles, when facing intelligent attacks.

\end{IEEEImpStatement}


\ifCLASSOPTIONpeerreview
\begin{center} \bfseries EDICS Category: 3-BBND
\end{center}
\fi
%
\IEEEpeerreviewmaketitle

\section{INTRODUCTION}
\IEEEPARstart{A}{s} the commercial drone industry is poised to take off in the near future, drones will play a key role in various applications such as parcel delivery, photography, precision agriculture, and power grid inspection, making our airspace more congested \cite{doole2020estimation}. At low altitudes, drones have to contend with obstacles as well as wildlife. In fact, bird attacks due to territorial behavior during nesting seasons are a well-documented threat to drones, with such a phenomenon curtailing Google's plans for last-mile drone delivery services \cite{Lagan2021}. Drones in surveillance operations might also be subjected to counter UAV measures \cite{kratky2018countering, 9764628} that include neutralization and capture. Current anti-drone approaches, such as nets mounted on another drone or a drone swarm, hand-held net cannons, and trained birds of prey, can be easily deployed at any feasible location. Even the killer drone \cite{park2021survey} is adopted to track and damage target drones with exceptional agility. Hence, drones equipped with the capability to maneuver and evade pursuers such as birds and malicious drones are required to improve their flight reliability and safety in such adversarial environments. This Pursuit-Evasion with Targeted Navigation (PETN) mission requires efficient perception and an online rapid-response control policy to continuously detect the positions of the pursuers and avoid their long-term chasing and even attacks while trying to accomplish their navigation task in the meantime. These requirements make evading pursuers for drones much more challenging compared to existing obstacle avoidance problems, as the pursuers have their own respective decision-making entities and can also perform continuous maneuvers.

\begin{figure}[!tbp]
      \centering
      \includegraphics[width=3.4in]{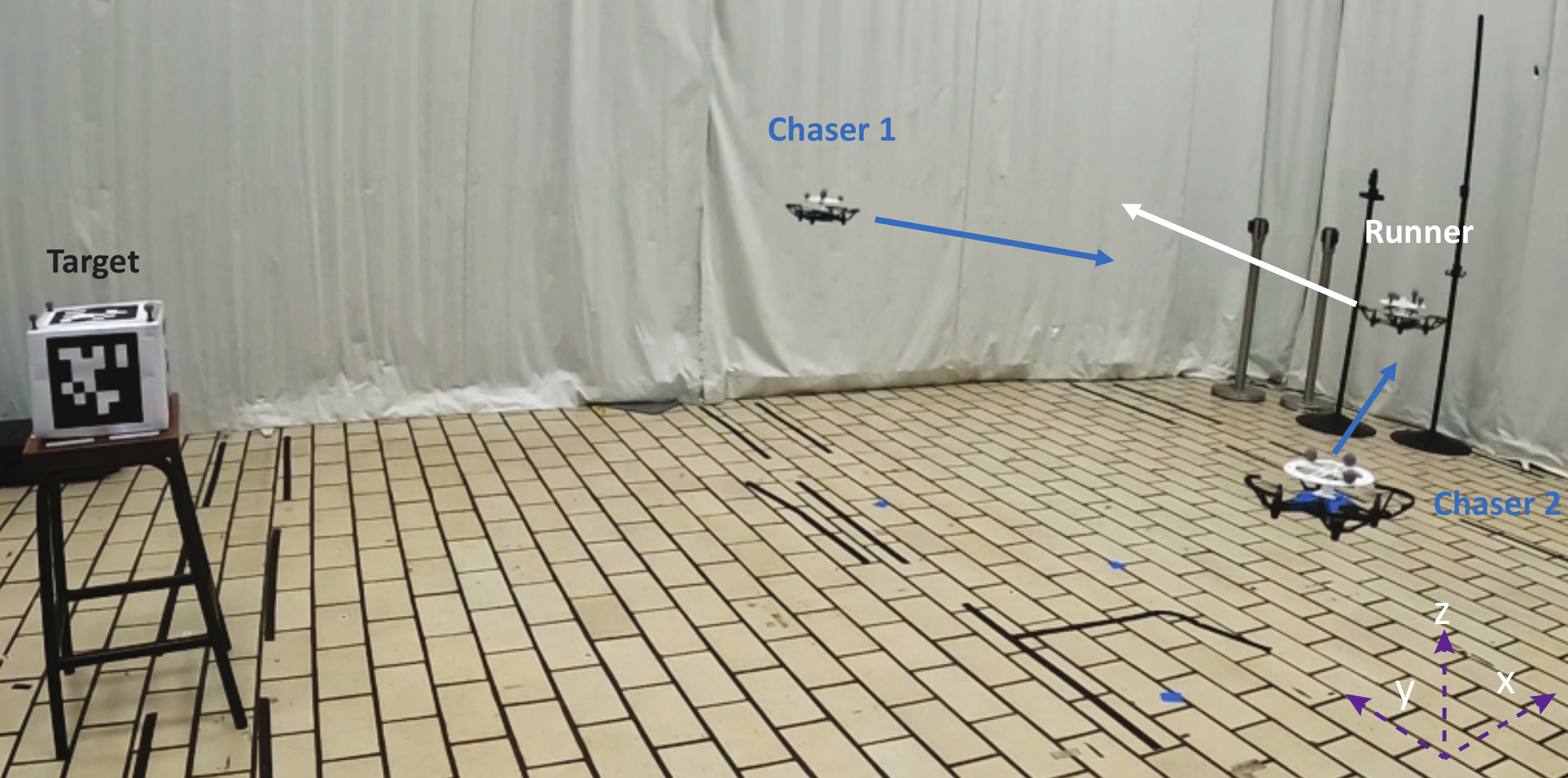}
      \caption{Evading multiple pursuers with learned policy. The runner (evader) is labeled as white, and the two chasers (pursuers) are labeled as blue. The target is a box with AprilTag.}
      \label{example}
\end{figure}

Nowadays, autonomous obstacle avoidance has become an indispensable component of drones' safe navigation. Massive obstacle avoidance methods have been proposed for drones' safe operation in complex environments \cite{zhou2019robust, Kaufmann2019, Falanga2020, song2021autonomous, xiao2021flying, Loquercio2021}. {However, most of them considered static cluttered obstacles, such as trees \cite{zhou2020ego} or buildings \cite{Loquercio2021}, whose positions would not change throughout the flight period. Even though some works proposed feasible approaches to avoid simple or ad-hoc dynamic obstacles (such as a moving gate \cite{Kaufmann2019}, balls being thrown \cite{Falanga2020, 9881875}, and moving pedestrians \cite{chen2022real}), the paths of the obstacles can be easily detected and predicted while in flight.} As such, these works only deal with unintelligent and unperceptive obstacles. The drone only needs to perform the required avoidance maneuvers at the predicted time. Furthermore, the self-learning capabilities of intelligent pursuers like birds can push drones to constantly approach their limits of computation and agility, which are intractable to solve with conventional obstacle avoidance approaches. {To effectively handle aforementioned threats, drones require higher agility and powerful learning capabilities.} Recently, the pursuit-evasion problem for drones has been investigated in the field of intelligent robotics \cite{9093032, vlahov2018developing, wang2020cooperative, 9387125, zhang2022game} which provides valuable insights to improve the maneuverability of drones. However, these works only focused on designing pursuit strategies without considering the safe navigation requirement from the evaders' side. Besides, they assumed that all dynamics of pursuers and evaders were pre-obtained, lacking consideration of evolving capabilities for both sides. Once the agents' behavior changes, the problem needs to be solved from scratch. To achieve PETN, a control policy must be designed to have sufficient agility and long-term evolution ability \cite{stanley2019designing} when performing navigation tasks, which existing traditional methods cannot provide.

Deep reinforcement learning (DRL) has recently attracted much attention for drone applications, such as obstacle avoidance \cite{xiao2021flying}, disturbance rejection \cite{o2022neural}, autonomous navigation \cite{10195084,zhao2023learning}, task coordination \cite{xiao2022collaborative, wu2021reinforcement}, drone racing \cite{kaufmann2023champion}, etc. DRL has shown outstanding performance over traditional numerical methods such as polynomial-based trajectory generation\cite{mellinger2011minimum} regarding the agility of drones\cite{song2021autonomous}. However, these DRL methods cannot be extended to solve the pursuit-evasion problem, as their policies are trained without the evolution of environments and adversarial agents. To address the PETN problem, especially multi-pursuit evasion with targeted navigation problem (MPETN) more efficiently, this paper proposes a novel \textbf{A}synchronous \textbf{M}ulti-\textbf{S}tage \textbf{D}eep \textbf{R}einforcement \textbf{L}earning (\textbf{AMS-DRL}) approach that can provide drones with rapid response and adaptive capabilities to evade pursuers.  

Motivated by nature-inspired pursuit-evasion phenomena where pursuers and chasers keep learning from each other with steady policy evolution, the proposed AMS-DRL is designed to evolve adversarial agents from multi-stage deep reinforcement learning, where the chaser drones (pursuers) and the runner drone (evader) are asynchronously trained in a bipartite graph way during different stages. With perfect information on position observation, the control policy for the runner drone is initially trained to reach the target box (the destination of the runner drone) in the cold-start learning stage ($S_0$). During the following stages, the chaser drones and the runner drone are trained against each other with one team fixed with a pre-trained policy from the previous stage; e.g., during Stage 1 ($S_1$), the chaser drones are trained to attack the runner drone while the runner drone is driven with the control policy from $S_0$. AMS-DRL can be seen as an unsupervised learning algorithm toward open-ended evolution \cite{stanley2019open}, which approximates the capability limits of agents consistently only if the adversarial agents reach a Nash equilibrium (NE) \cite{holt2004nash}. We demonstrate that the resulting policies can provide drones with evasive maneuvers for self-preservation, opening the possibility of using DRL for solving multi-pursuit evasion problems. Our main contributions are summarized as follows: 

(1) {We formulate the MPETN problem into an adversarial mixed game and present a DRL-based approach involving multiple drones.} A novel learning-based approach named AMS-DRL is proposed to provide drones with sufficient maneuverability and long-term evolution ability to evade multiple pursuers during navigation.

(2) The proposed AMS-DRL is proven to converge to optimal policies based on game theory. Through extensive simulations, we demonstrate the superiority of our navigation policy over baseline methods such as PPO and the artificial potential fields (APF) method against collaborative pursuers.

(3) Sim2Real and physical experiments under different spatial geometry configurations are conducted to verify the effectiveness of the trained navigation policy in real flights.

The rest of this paper is organized as follows. Section II summarizes existing methods related to our work. Section III provides several preliminary works, including the quadrotor drone dynamics and RL framework. Section IV presents the problem formulation and the RL description of MPETN. Section V describes the AMS-DRL approach in detail. Experiments and results are presented and discussed in Section VI, and we conclude our work in the final section.

\section{Related Works} \label{section 2}
\subsection{Obstacle Avoidance}
Obstacle avoidance aims to achieve designed objectives (e.g., reaching the destination) while subject to non-intersection or non-collision constraints. Obstacle avoidance is a basic requirement for intelligent drones to fly through cluttered environments. From the perspective of how drones generate control commands, existing obstacle avoidance can be categorized into indirect methods and end-to-end methods. Indirect methods \cite{Falanga2020,zhou2019robust, zhou2020ego,zhou2021raptor,Quan2021} generally divide the obstacle avoidance problem into several subtasks, namely \textbf{perception}, \textbf{localizing} and \textbf{planning}. The states of obstacles, such as shape, position, and velocity, are detected and localized before a maneuver is performed. Once obstacles are located, the drone can generate a feasible path or take action to avoid the obstacles. Within the indirect methods, planning can be further divided into two categories: one is offline methods based on high-resolution maps and pre-known position information, such as Dijkstra's algorithm \cite{Dijkstra1959ANO}, A$^\star$ \cite{4082128} and sequential convex optimization \cite{augugliaro2012generation}; the other is online methods based on real-time decision-making. Online methods can be further categorized into online path planning \cite{zhou2019robust, zhou2020ego, zhou2021raptor} and artificial potential field (APF) methods \cite{chen2016uav, Falanga2020}. Currently, due to the advantages of optimization and prediction capabilities, online path planning methods have become the preferred choice for drone obstacle avoidance. For instance, Zhou Boyu \textit{et al.}\cite{zhou2019robust} introduced Fast-Planner for a drone to perform high-speed flight in an unknown cluttered environment. The key contributions are a robust and efficient planning scheme incorporating path searching, B-spline optimization, and time adjustment to generate feasible and safe trajectories for drones' obstacle avoidance. However, these online path planning methods necessitate onboard complex optimization computation to obtain a local or global optimal trajectory. In contrast to online path planning, the APF methods require fewer computation resources and can well cope with dynamic obstacle avoidance using limited sensor information. Falanga \textit{et al.} \cite{Falanga2020} developed an efficient and fast control strategy based on APF to avoid fast-approaching dynamic obstacles. However, the repulsive forces computed only reach substantial values when the obstacle is very close, which may lead to unstable and jagged behavior. Besides, APF methods are heuristic methods that cannot guarantee global optimization and robustness for drones. Hence, it is not ideal to adopt APF methods for drones to evade pursuers.

Compared to the indirect methods, end-to-end methods \cite{anwar2018navren, dai2020automatic,xiao2021flying, Loquercio2021} exploit reinforcement learning (RL) to map perception to actions directly. Reinforcement learning \cite{Andrew1998} is a technique for mapping the state (observation) space to the action space in order to maximize a long-term return with given rewards. A typical RL model features agents, the environment, reward functions, action, and state space. The policy model achieves convergent status via constant interactions between the agents and the environment, where the reward function guides the training process. Malik \textit{et al.} \cite{anwar2018navren} presented an end-to-end reinforcement learning approach called NAVREN-RL to navigate a quadrotor drone in an indoor environment with expert data and knowledge-based data aggregation. The reward function in \cite{anwar2018navren} was formulated from a ground truth depth image and a generated depth image. Loquercio \textit{et al.} \cite{Loquercio2021} developed an end-to-end approach that can autonomously guide a quadrotor drone through complex wild and human-made environments at high speeds with purely onboard visual perception and computation. The neural network policy was trained in a high-fidelity simulation environment with massive expert knowledge data. While the aforementioned end-to-end methods provide a straightforward way to generate obstacle avoidance policies for drones, they require massive training data to obtain acceptable generalization capabilities. Meanwhile, without the evolutionary environment, it is challenging for these designed neural network policies to update their weights to evade intelligent pursuers. 

\subsection{Pursuit-Evasion Game}
Pursuit-evasion games (PEGs) have been widely studied in the fields of game theory and control theory. PEGs involve capturing mobile agents (evaders) with one or more pursuers. Back to the 1950s, Isaacs R. provided a comprehensive introduction to the mathematical theory of pursuit-evasion games\cite{isaacs1999differential}, which is considered a classic reference in this field. Fisac Jaime F. \textit{et al.} \cite{fisac2015pursuit} investigated the problem of the pursuit-evasion-defense differential game with dynamic obstacles in constrained environments. Xu Fang \textit{et al.} \cite{9093032} studied the multi-pursuer single-evader PEG with an evader moving faster than the pursuers in a constraint-free environment, and proposed a distributed pursuit algorithm to enable pursuers to form an encirclement and approach the faster evader, which was applied and verified in mobile ground robots. These works provide traditional game/control-theoretic approaches for solving PEG problems in dynamic, constrained, or obstacle-free environments. Most of the existing traditional approaches require simplified dynamics or environment assumptions and are challenging to extend to general pursuit-evasion scenarios. 

Recently, RL has been adopted to ease strict assumptions for pursuit-evasion problems \cite{vlahov2018developing, wang2020cooperative, 9387125, zhang2022game, 9716772, kokolakis2022safety}. A distributed cooperative pursuit strategy is proposed in \cite{wang2020cooperative} to address the multiple pursuers-one superior evader problem with RL and efficient communication topology networks. A curriculum learning approach in \cite{9387125} is developed to train a decentralized multi-agent pursuit policy to capture mobile targets. In \cite{zhang2022game}, a target prediction network (TP Net) is integrated into the multi-agent deep deterministic policy gradient (MADDPG) to address the target PEG problem of drones in the obstacle-filled environment, where high-quality 3D simulation scenarios are developed. The attention-enhanced neural network was introduced in \cite{9716772} to address predator-prey games. However, most of these RL-based approaches assume fixed or simple evader behaviors and rarely consider the NE of the PEGs. Moreover, they focus only on the pursuit side and ignore the safety of evaders. In this paper, we formulate the safe targeted navigation problem for the evader drone and address it by finding the NE with AMS-DRL.

\subsection{Reinforcement Learning with Self-play}
With self-play \cite{heinrich2015fictitious,heinrich2016deep,silver2018general}, RL has achieved remarkable success in adversarial environments (or competitive games), such as DeepMind's AlphaZero \cite{silver2018general}, AlphaStar \cite{vinyals2019grandmaster} and OpenAI Five \cite{berner2019dota}. Self-play uses the agent’s current or past ``selves" as opponents to train agents to automatically generate curricula, which can be a powerful algorithm for learning behaviors within high-dimensional continuous environments. Silver \textit{et al.} \cite{silver2018general} demonstrated the superhuman performance of AlphaZero, a general RL algorithm with self-play, in many challenging perfect-information games (such as chess, shogi, and Go) without domain knowledge. To address large-scale imperfect-information games with RL, Heinrich \textit{et al.} \cite{heinrich2016deep} introduced Neural Fictitious Self-Play (NFSP) via a combination of Fictitious Self-Play (FSP) \cite{heinrich2015fictitious} and neural networks, and demonstrated it can approach NE without prior domain knowledge over real-world scale games. Following that, Vinyals \textit{et al.} \cite{vinyals2019grandmaster} provided the first Grandmaster level player-AlphaStar in the real-time strategy game StarCraft II with complex, imperfect information and long-time horizons. AlphaStar features a Prioritized Fictitious Self-Play (PFSP) algorithm, which adopts a match-making mechanism to provide a better learning signal instead of uniformly playing against all previous ``selves" to improve the training efficiency. In 2019, OpenAI Five \cite{berner2019dota} first defeated the world champions in the Dota 2 game with proposed self-play experience buffers and surgery techniques, which achieved large-scale distributed learning and long-time training. These methods paved the way for training agents in complicated adversarial environments and provided them with superhuman performance against opponents. Furthermore, OpenAI found that multi-agent self-play can induce distinctive emergent strategies toward open-ended growth in complexity \cite{baker2019emergent}, which helped solve evolution-required decision-making problems. However, all the aforementioned methods only process zero-sum two-team games where the rewards of two adversarial teams can be flipped, i.e., the sum of rewards for two teams remains zero during training. In fact, it is determined by the inherent property of self-play that exploits ``selves" as opponents during training, which means two adversarial teams' observations and rewards must be symmetric. 

Nevertheless, pursuit-evasion for drones is not a zero-sum game since the runner drone is required to reach the target while the chaser drones aim to pursue the runner drone. The success of reaching the target will not bring a loss to the chasing team. Meanwhile, the observation spaces for the runner agent and chaser agent are not symmetric since the chaser drones cannot perceive the location of the target. In our work, we extend the capabilities of DRL to address the pursuit-evasion problem for drones via AMS-DRL. AMS-DRL provides a cold-start learning stage for the runner drone to reach the target and asynchronous training for two adversarial teams at multiple stages/phases. In each phase, agents can compete with their opponents driven by the policy from the prior stage/phase. The limitation of self-play is addressed via the asynchronous training scheme toward an approximate NE.




\section{Preliminaries}
\subsection{Quadrotor Drone Dynamics}
We first present the quadrotor drone dynamics for agent modeling. Regardless of the wind disturbance and the aerodynamic drag, the 6-degree-of-freedom (6DOF) dynamics of a quadrotor drone can be formulated as:
\begin{subequations}
\setlength{\arraycolsep}{0.0em}
\begin{eqnarray}
   \bm{\dot {p}}&{}={}&\bm{v}\\
   \bm{\dot {v}}&{}={}& \bm{R}(\bm{q}) \odot \bm{f} + \bm{g}\\
   {\bm{\dot q}}&{}={}& \frac{1}{2}\bm{\Omega}(\bm{\omega}_B) \cdot \bm{q} \\
   \dot {\bm{\omega}_B} &{}={}& \bm{J}^{-1}(\bm{\tau} - \bm{\omega}_B \times \bm{J} \bm{\omega}_B)
\end{eqnarray}
\setlength{\arraycolsep}{5pt}
\end{subequations}
with state $\bm{x} = [\boldsymbol{p}^T, \boldsymbol{v}^T, \boldsymbol{q}^T]^T$ and input $\boldsymbol{u} = [f, \omega_x,\omega_y,\omega_z]^T$, where ${\boldsymbol{p} = [x,y,z]^T}$ and ${\boldsymbol{v} = [v_x, v_y, v_z]^T}$ are the position and velocity of the drone. The $z$ axis is opposite to the direction of Earth gravity $\boldsymbol{g}$. $\boldsymbol{q} = [q{_0},q_1,q_2,q_3]^T$ is the unit quaternion which is used to describe the quadrotor's attitude, while ${\boldsymbol{\omega}}_B = [\omega_x,\omega_y,\omega_z]^T$ denotes the body rate (roll, pitch, and yaw, respectively). $\boldsymbol{f}=[0,0,f]^T$ is the mass-normalized thrust vector with $f = \sum_{i=1}^{4} f_i/m$. Here, $f_i$ are the four individual motor thrusts. $\bm{R(q)}$ defines the rotation matrix from the body frame $O_B$ to the world inertial frame $O_W$, and it is a function of $\boldsymbol{q}$. $\odot$ and $\times$ denote the inner production and cross production, respectively. $\boldsymbol{g} = [0,0,-g_z]^T$ with $g_z = 9.81m/s^2$ is the gravity acceleration on Earth. $\bm{J}$ is the inertia matrix of the drone and $\bm{\tau}$ is the generated torque. The skew-symmetric matrix $\boldsymbol{\Omega}(\boldsymbol{\omega}_B)$ is given by
\begin{equation} \label{skew-sym}
\bm{\Omega}(\bm{\omega}_B) = \left[ {\begin{array}{*{20}{c}}
0&{ - {\omega _x}}&{ - {\omega _y}}&{ - {\omega _z}}\\
{{\omega _x}}&0&{{\omega _z}}&{ - {\omega _y}}\\
{{\omega _y}}&{ - {\omega _z}}&0&{{\omega _x}}\\
{{\omega _z}}&{{\omega _y}}&{ - {\omega _x}}&0
\end{array}} \right]
\end{equation}

The body's momentum will not be examined in this model since our method focuses on the high-level control policy and the deployed platform can precisely track the angular rate commands, allowing angular dynamics to be ignored. The inputs are constrained within limited ranges, i.e.,
\begin{equation} \label{ctl_contraint}
    f_{min} \leq f \leq f_{max} ~~~~ \omega_{min} \leq ||\bm{\omega}_B|| \leq \omega_{max}  
\end{equation}

\subsection{Reinforcement Learning}
Reinforcement Learning excels in training models to make decisions in complex, unpredictable, and ever-changing environments. A typical RL model features agents, the environment, reward functions, action, and state spaces. The environment and the agents always interact during training and execution. An action is carried out by the agent once it obtains a state from the environment. The environment acknowledges this action and rewards the agent based on the outcome of the state change. This ongoing interaction ends when the agent reaches the desired state, the maximum number of time steps has been completed, or a predefined termination condition is triggered. The agent's objective throughout this process is to maximize the expected cumulative reward.

\subsubsection{Markov Decision Process (MDP)}
The RL problem can be modeled with a Markov Decision Process (MDP) \cite{sutton2018reinforcement} which is defined by a tuple $M:\langle \sset, \mathcal{O}, \aset, \Tfun, \rset, \D \rangle$. $\sset$ denotes a set of possible states of agents. At time step $t$ and current state $s_t \in \sset$, the agent obtains an observation $o_t \in \mathcal{O} \subseteq \sset$ and selects an action $a_t \in \aset$. $\Tfun(s_{t+1}\mid s_t,{a}_t): {\sset \times \aset \times \sset \mapsto [0,1]}$ is the transition probability of the environment moving from the current state $s_t$ to the next state $s_{t+1}$ after taking action ${a}_t$. $\rset: \sset \times \aset \mapsto \mathbb{R}$ refers to the reward function evaluating the bounded instantaneous reward $r_t \in [\rmin, \rmax]$. Each agent will receive an instantaneous reward $r_t := \rset(s_t, a_t)$ after taking action $a_t$ at $s_t$.

The agent starts from the state $s_0$ sampled from an initial state distribution $s_0 \sim \issetdef$. At each time step $t$, the agent is guided by a stochastic policy $\p(a_t \mid o_t) : \mathcal{O} \times \aset \mapsto [0,1]$. Given a policy $\pi$ and a state $s_t$, the state value function $V: \sset \mapsto \rset$ can be defined as:
\begin{equation}
    V^{\pi}(s_t) := \mathop{\mathbb{E}_{\pi}\left[ \sum_{k = 0}^{k = T-t}{\D}^{k}{r_{t+k}\mid{s_t}}\right]}
\end{equation}
where $\Ddef$ is the discount factor and $T$ is the final step of an episode. Similarly, the state-action value function $Q: \sset \times \aset \mapsto \rset$ can be calculated as a Bellman function:
\begin{subequations}
\setlength{\arraycolsep}{0.2em}
\begin{eqnarray}
    Q^{\pi}(s_t, a_t) & = & \mathop{\mathbb{E}_{a\sim \pi}\left[ r_t + \D V^{\pi}(s_{t+1})\mid{s_t, a_t}\right]}\\
                    & = & r_t + \D \cdot \mathop{\mathbb{E}_{a\sim \pi}\left[ Q^{\pi}(s_{t+1}, a_{t+1})\right]}
\end{eqnarray}
\setlength{\arraycolsep}{5pt}
\end{subequations}

In policy-based methods, the objective of the RL is to find an optimal policy $\pi_{\theta}^*: \mathcal{O} \mapsto \aset$ to maximize the cumulative reward along a state-action trajectory $\tau:= \{s_0, a_0, s_1, a_1, ...\}$ via adjusting the weights $\theta$ of the parameterized policy $\p_{\theta}$ over finite time steps (a training episode). The objective function can be formulated as
\begin{equation}
{\mathcal{L}({\p}_{\theta})} = \mathop{\mathbb{E}_{{s_0 \sim d_0},{\tau}\sim \pi_\theta}\left[ \sum_{t = 0}^{t = T}{\D}^{t}{r_{t}\mid{s_0}}\right]}
\end{equation}
The optimal policy is obtained by:
\begin{equation}
{{{\p}}^{*}_{\theta}} = \argmax_{\theta}{\mathcal{L}({\p}_{\theta})}
\end{equation}
\subsubsection{Policy Gradient Theorem}
Within policy-based methods, the state value function $V^{\pi_\theta}(s_t)$ is always approximated by state-action value function $Q^{\pi_\theta}(s_t,a_t)$. In the actor-critic (AC) theory \cite{konda1999actor}, the policy network $\pi_\theta(a_t \mid o_t)$ is the actor network and the state value function approximator $Q^{\pi_\theta}(s_t, a_t)$ is the critic network. The optimal policy $\pi_\theta^*$ can be obtained via policy iteration along the gradient of the critic network (performance objective $\mathcal{L}({\p}_{\theta})$) according to the policy gradient theorem \cite{sutton1999policy}. The policy update with learning rate $\alpha_t$ is:
\begin{equation}
    \nabla_\theta{\mathcal{L}({\p}_{\theta})} = \mathop{\mathbb{E}_{{s_t \sim \Tfun },{a_t}\sim \pi_\theta}\left[\nabla_\theta \log \pi_\theta(a_t \mid o_t)  Q^{\pi_\theta}(s_t, a_t)\right]}
\end{equation}
\begin{equation} \label{sga}
    \theta_{t+1} = \theta_t + \alpha_t \nabla_{\theta_t}{\mathcal{L}({\p}_{\theta_t})}
\end{equation}

To reduce the variance of the gradient during the training, an advantage function \cite{sutton2018reinforcement} $A^\pi(s_t, a_t) := Q^\pi(s_t, a_t)-V^\pi(s_t)$, which measures the performance of an action, is commonly adopted in the stochastic gradient ascent algorithm (\ref{sga}). However, an estimator of the advantage function $\hat{A}^{\pi_\theta}(s_t, a_t)$ is considered in most existing algorithms since the accurate value of the advantage function is hard to obtain in real time. Hence, the variance-reduced gradient can be described as follows:
\begin{equation} \label{Adv}
    \nabla_\theta{\mathcal{L}({\p}_{\theta})} = \mathop{\mathbb{E}_{{s_t \sim \Tfun },{a_t}\sim \pi_\theta}\left[\nabla_\theta \log \pi_\theta(a_t \mid o_t)  \hat{A}^{\pi_\theta}(s_t, a_t)\right]}
\end{equation}

\subsubsection{Proximal Policy Optimization Algorithm}
The proximal policy optimization (PPO) algorithm \cite{schulman2017proximal} is an efficient and robust actor-critic algorithm that uses a ``clipped” surrogate objective function and on-policy learning. In PPO, the clipped surrogate objective function is formulated with \textit{advantage estimates} $\hat{A}^{\pi_\theta}_t :=\hat{A}^{\pi_\theta}(s_t, a_t)$ over fixed $T$ timesteps:
\begin{equation} \label{ppo}
    {\mathcal{L}^{clip}({\p}_{\theta})} = \mathop{\mathbb{E}_{t}\left[\min(\mu_t(\theta)\hat{A}^{\pi_\theta}_t, clip(\mu_t(\theta), 1-\epsilon, 1+\epsilon)\hat{A}^{\pi_\theta}_t))\right]}
\end{equation}
where $\mu_t(\theta)$ denotes the policy probability ratio $\mu_t(\theta) = \frac{\pi_{\theta}(a_t \mid o_t)}{\pi_{\theta_{old}}(a_t \mid o_t)}$ and $\epsilon$ is a hyperparameter controlling the clip range. With the policy probability ratio and the clip scheme, the PPO forms a low bound on the performance of the policy $\pi$ and keeps the policy updated without far deviation from the old policy. In addition, the advantage estimates $\hat{A}^{\pi_\theta}_t$ in PPO enable efficient reuse of training data.

\section{Multi-Pursuit Evasion with Targeted Navigation Problem}
\subsection{Problem Formulation}
In this section, we introduce the idea of the MPETN problem for drones and formulate it into an adversarial mixed game described with MDPs, which lies the foundation of our control algorithm AMS-DRL. Note that our multi-pursuit evasion problem in this paper considers one runner drone against a chaser drone swarm with the same speed. The formulated problem and the proposed algorithm can be extended to a large-scale scenario where massive chaser drones of a finite number are involved. We assume that the runner only knows the relative positions of the chasers and the target, while all chasers can obtain the relative positions of the runner and other chasers. The positions of all agents and the target are randomly placed with {a uniform} distribution in a constrained sample space. The system dynamics are unknown to each agent, i.e., the states cannot be predicted directly by solving system differential functions.
\subsubsection{MDP in Adversarial Games}
The pursuit evasion problem in this paper involves two adversarial teams, namely a runner team (single-agent) $\mathcal{M}^r:\langle \sset, \mathcal{O}^r, \aset^r, \Tfun, \rset^r, \D \rangle$ and a chaser team (multi-agent) $\mathcal{M}^c:\langle \mathcal{N}, \sset, \{\mathcal{O}^c_i\}_{i \in \mathcal{N}}, \{\mathcal{\aset}^c_i\}_{i \in \mathcal{N}}, \Tfun, \{\mathcal{\rset}^c_i\}_{i \in \mathcal{N}}, \D \rangle$, where $\sset := \sset^c \cup \sset^r $ defines the state space covered by all chaser agents $\sset^c$ and the runner agent $\sset^r$, $\mathcal{N} = \{1,...,N\}$ denotes the set of finite $N \ge 1$ chaser agents, $\mathcal{O}^c_i$, ${\aset}^c_i$, and $\mathcal{\rset}^c_i$ are the observation space, action space, and reward function of the chaser agent $i$, respectively. Let $\aset := {\aset}^c_1 \times \dots \times {\aset}^c_N \times \aset^r$, then $\Tfun:\sset \times \aset \times \sset \mapsto [0,1] $ defines the transition probability from state $s_t \in \sset$ to next state $s_{t+1} \in \sset$ with joint action $a_t \in \aset$.
The goal of the runner agent is to reach the target $g$ and avoid the attack from the chaser team while the chaser team is trying to crash down the runner agent, i.e., to minimize the cumulative reward of the runner agent. Hence, the pursuit-evasion problem for drones can be formulated as an adversarial mixed game with two adversarial teams $\{\mathcal{M}^r, \mathcal{M}^c\}$. The objective of the pursuit-evasion problem can be described as:
\begin{subequations} \label{mixed_game}
\setlength{\arraycolsep}{0.2em}
\begin{eqnarray}
    \pi_{\theta}^{r*} &=& \argmax_{a^r_{t} \sim \p_{\theta}^{r}}{\mathcal{L}({\p_{\theta}^r}\mid{\p_{\theta'}^{c*}})} \\
    s.t. ~~ \p_{\theta'}^{c*} &=& \argmin_{\{a^c_{i,t}\} \sim \p_{\theta'}^{c}}{\mathcal{L}({\p}_{\theta}^r)} \label{mixed_game_b} \\
    \dot{\bm{x}} &=& g(\bm{x}, \bm{u}) \\
    \bm{u}_{min} &\leq& \bm{u} \leq \bm{u}_{max} 
\end{eqnarray}
\end{subequations}
where $\pi^c := \prod_{i\in \mathcal{N}}\pi_i^c(a^c_i \mid{s})$ is the joint chaser policy and $\p^r$ is the runner policy. Subscripts $\theta$ and $\theta'$ mean the {policies} $\pi^r$ and $\pi^c$ are parameterized with different neural networks. $g(\cdot)$ is the dynamics constraint, and $\bm{u}_{min}$ and $\bm{u}_{max}$ are the input bounds. In the game theory \cite{fudenberg1991game}, the optimal solution of a game is always achieved at a NE. The NE of the adversarial mixed game $\{\mathcal{M}^r, \mathcal{M}^c\}$ is defined as follows.

\begin{definition}[Nash Equilibrium of Adversarial Mixed Game]
A Nash Equilibrium of the adversarial mixed game described with MDPs $\mathcal{M}^r:\langle \sset, \mathcal{O}^r, \aset^r, \Tfun, \rset^r, \D \rangle$ and $\mathcal{M}^c:\langle \mathcal{N}, \sset, \{\mathcal{O}^c_i\}_{i \in \mathcal{N}}, \{\mathcal{\aset}^c_i\}_{i \in \mathcal{N}}, \Tfun, \{\mathcal{\rset}^c_i\}_{i \in \mathcal{N}}, \D \rangle$ is a joint policy $\pi^*:= \{\pi_1^{c*}, \dots, \pi_N^{c*}, \pi^{r*}\}$, such that for any $s \in \sset$ and $i \in \mathcal{N} \cup \{r\}$,
\begin{equation} \label{NE_def}
    V^i_{\p_i^*, \p_{-i}^*}(s) \ge V^i_{\p_i, \p_{-i}^*}(s), ~~for~any ~\pi_i
\end{equation}
where $\{r\}$ {denotes} the runner agent, and $-i$ means all other agents except for agent $i$.
\end{definition}

\begin{theorem}[]
\label{existence}
For the adversarial mixed game $\{\mathcal{M}^r, \mathcal{M}^c\}$, there always exists a unique NE, and this unique NE is the optimal policy $\pi_{\theta}^{r*}$ for the runner agent, i.e., the solution of the pursuit-evasion problem (\ref{mixed_game}).
\end{theorem}

\begin{proof}
Firstly, we prove the existence and uniqueness of NE of the adversarial mixed game $\{\mathcal{M}^r, \mathcal{M}^c\}$. Since the chaser team $\mathcal{M}^c$ consists of homogeneous agents with a common reward function and same policies, i.e., $\pi_{i\in\mathcal{N}}^c = \pi_1^c$, we can regard the chaser team as one player $c$. Therefore, the adversarial mixed game $\{\mathcal{M}^r, \mathcal{M}^c\}$ shrinks into a two-player mixed game with player $\{c,r\}$. Due to the control constraints (\ref{ctl_contraint}), the action space of the agents is finite. According to \textbf{Theorem 1.1} in \cite{fudenberg1991game}, there always exists an NE for finite strategic games. Suppose there are two NE strategies $\pi^{\alpha}$ and $\pi^{\beta}$, based on (\ref{NE_def}), we have $V^i_{\p_i^{\alpha}, \p_{-i}^*}(s) \ge V^i_{\p_i^{\beta}, \p_{-i}^*}(s)$ and $V^i_{\p_i^{\beta}, \p_{-i}^*}(s) \ge V^i_{\p_i^{\alpha}, \p_{-i}^*}(s)$ $\Rightarrow$ $V^i_{\p_i^{\alpha}, \p_{-i}^*}(s) \equiv V^i_{\p_i^{\beta}, \p_{-i}^*}(s)$ at any time. Hence, $\pi^{\alpha} \equiv \pi^{\beta}$ holds, and it is the unique NE of the adversarial mixed game.

From \textbf{Definition 1}, let $i \in \mathcal{N}$, we have $V^i_{\p^{c*}, \p^{r}}(s) \ge V^i_{\p^c, \p^{r}}(s) = -({\mathcal{L}({\p}_{\theta}^r)}-r_T)$ for any $\pi^c$, where $r_T$ is the positive constant terminal reward of the runner agent. It follows that $-V^i_{\p^{c*}, \p^{r}}(s) \leq ({\mathcal{L}({\p}_{\theta}^r)}-r_T) < {\mathcal{L}({\p}_{\theta}^r)}$. Therefore, $\p^{c*}$ is the desired $\p_{\theta'}^{c*}$ in (\ref{mixed_game_b}). Similarly, let $i \in \{r\}$, we have $V^i_{\p^{r*}, \p^{c*}}(s) \ge V^i_{\p^r, \p^{c*}}(s) = {\mathcal{L}({\p}_{\theta}^r\mid {\p}_{\theta'}^{c*}})$ for any $\p^r$ under $\p^{c*}$, i.e., ${\p}_{\theta}^{r*} = \p^{r*}$. This completes the proof.
\end{proof}

\begin{remark}
Note that the terminal reward of the runner agent $r_T$ is assumed to be a positive constant when the runner agent reaches the target. Considering an indoor environment, there is a probability of the runner agent colliding with the wall, which gives a negative reward $r'_T$. In this case, we assign an opposite reward $-r'_T$ to each chaser agent since the runner agent is trained to fly away from the wall during the initial stage and the collision can be attributed to the pursuit of the chaser agents. Therefore, \textbf{Theorem 1} holds even though the terminal state of the runner drone is the wall collision.
\end{remark}

\subsection{RL Description of MPETN}
In this section, we describe the MPETN problem from the RL formulation. The environment setup, state space, observation space, action space, state transition function, and reward function are described in detail.
\subsubsection{Environment Setup}
To train and evaluate the policies of chaser and runner agents, a high-fidelity simulation environment is created using the Unity game engine. Our simulation environment is an obstacle-free, constrained indoor room, which is illustrated in Fig. \ref{figure1}. During training, all three drone agents and the red target box will spawn in random locations every episode to prevent any agent from learning undesirable behaviors such as flying toward a fixed spawn point.

\begin{figure}[!tbp]
\setlength{\abovecaptionskip}{0.cm}
\setlength{\belowcaptionskip}{-0.cm}
\centering
\subfigure[Top view]{
\includegraphics[width=3.1in,trim=0 0 0 0,clip]{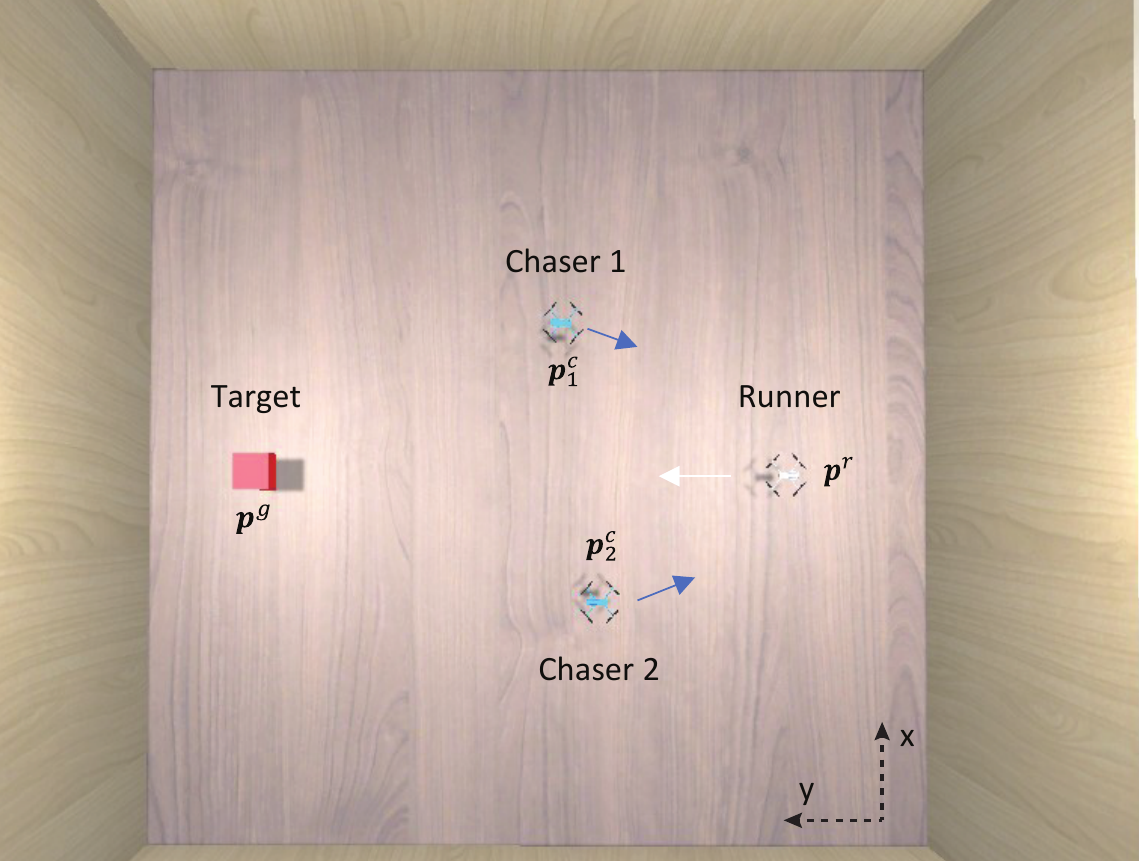}
}
\quad
\subfigure[Side view]{
\includegraphics[width=3.1in,trim=0 0 0 0,clip]{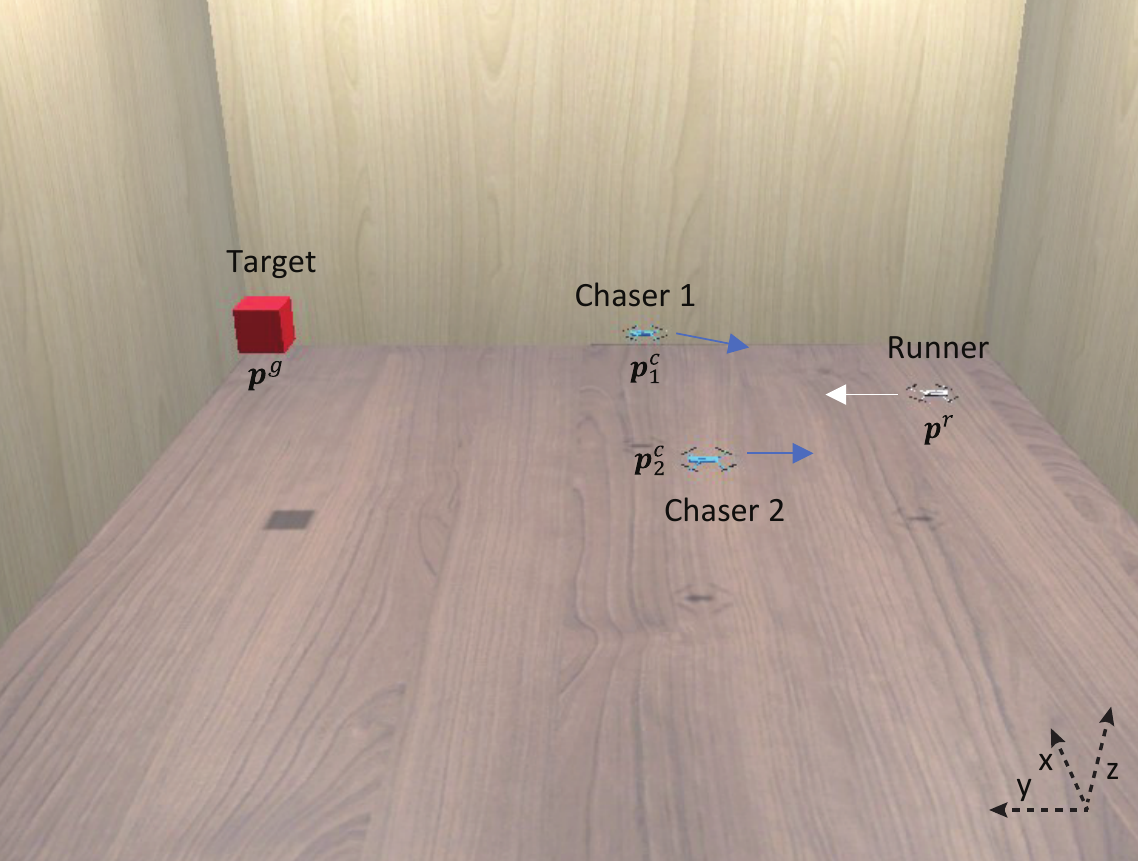}
}

      \caption{The simulation environment for the 3D multi-pursuit evasion with targeted navigation scenario, which consists of two blue chaser drones (pursuers) and one white runner drone (evader) flying towards the target (red box). {All simulated objects are {bound} with box colliders ($L\times W \times H$) for collision detection.} The chasers aim to crash down the runner while the runner is required to safely reach the target.}
      \label{figure1}
\end{figure}

The components of each object are outlined by {Unity’s box collider}, and all objects are tagged as target, chaser, and runner, respectively. The wall, floor, and ceiling are tagged as wall colliders. These tags allow for the detection of collisions with each other, with different rewards.

\subsubsection{State Space}
Since the drone is an under-actuated system, in a three-dimensional (3D) space, the state of a drone $i \in \mathcal{N} \cup \{r\}$ can be represented as $s_i := [x_i,y_i,z_i,\psi_i]^T\in \mathbb{R}^4$, where $x_i,y_i,z_i$ are the elements of position $\boldsymbol{p}_i$, and $\psi$ is the heading angle. We denote the positions of target, runner, and chasers as $\boldsymbol{p}^{g}$, $\boldsymbol{p}^{r}$, and $\boldsymbol{p}^{c}$, respectively. The global state can be represented as $s:=\{\boldsymbol{p}^{g}, \{s_i\}_{i \in \mathcal{N} \cup \{r\}}\}$.
 
\subsubsection{Observation Space}
In a partially observed scenario, the agents draw actions with a certain probability from observations. To improve the generalization of the trained model, relative positions are considered in the observation space. The observation for the runner agent $o^r$ consists of {the} relative position towards the target $\boldsymbol{p}^{rg} = \boldsymbol{p}^{g} - \boldsymbol{p}^{r}$ and relative positions towards chaser agents $\{\boldsymbol{p}_j^{rc}\}_{j\in \mathcal{N}}$, where $\boldsymbol{p}_j^{rc} = \boldsymbol{p}_j^{c} - \boldsymbol{p}^{r}$. The concatenated observation is $o^r = [\boldsymbol{p}^{rg}, \{\boldsymbol{p}_j^{rc}\}_{j\in \mathcal{N}}] \in \mathbb{R}^{(N+1)\times 3}$. 
For chaser agents $\{c_i\}_{i \in \mathcal{N}}$, they share the positions with each other and can track the position of the runner all the way. Hence, the observation of chaser agent $c_i$ is $o^c_i = \left[ \{\boldsymbol{p}_{ij}^{cc}\}_{j\in \mathcal{N} \setminus \{i\}}, \boldsymbol{p}_i^{cr}\right] \in \mathbb{R}^{N\times 3}$, where $\boldsymbol{p}_{ij}^{cc}$ is the relative position from the chaser $c_i$ to the chaser $c_j$, i.e., $\boldsymbol{p}_{ij}^{cc} = \boldsymbol{p}_{j}^{c} - \boldsymbol{p}_{i}^{c}$ and $\boldsymbol{p}_i^{cr}$ is the relative position from the position of chaser $c_i$ to the position of runner $r$, i.e., $\boldsymbol{p}_{i}^{cr} = \boldsymbol{p}^{r} - \boldsymbol{p}_{i}^{c}$.

\subsubsection{Action Space}
In this work, the action spaces for the runner drone $a^r$ and the chaser drone $a^c$ are the same. The selected actions $a^c = [\hat v_x^{B,c}, \hat v_y^{B,c}, \hat v_z^{B,c}, \hat \omega_z^c] \in \mathbb{R}^4$ and $a^r = [\hat v_x^{B,r}, \hat v_y^{B,r}, \hat v_z^{B,r}, \hat \omega_z^r] \in \mathbb{R}^4$ are four desired velocity commands generally used in the Robot Operation System (ROS), where $\hat v_x^B \in (-v_{max},v_{max})$, $\hat v_y^B \in (-v_{max},v_{max})$, $\hat v_z^B \in (-v_{max},v_{max})$ and $\hat \omega_z \in (-\omega_{max},\omega_{max})$ are the desired velocities of the quadrotor in the body frame $O_B$. The action space $\aset$ consists of these four continuous actions that drive drones to maneuver in the obstacle-free 3D space.

\subsubsection{State Transition Function}
Since we consider the high-level controller based on the selected velocity commands, the quadrotor dynamics (1a)-(1c) can be simplified as:
\begin{equation} \label{simsys}
    \dot{\bm{x}} = \bm{g}(\bm{x}) \cdot \bm{u} + \bm{\omega}
\end{equation}
where $\bm{x} = [x,y,z,\psi]^T$, $g(\bm{x})$ is the transition matrix, and $\bm{u} =[\hat v_x^B, \hat v_y^B, \hat v_z^B, \hat \omega_z]^T$. $\bm{\omega}$ represents the process noise and uncertainties, which satisfies the Gaussian distribution $p(\bm{w})\sim\aleph (0,\bm{Q})$ with $\bm{Q} = diag\{\sigma _{x}^2,\sigma _{y}^2, \sigma _{z}^2, \sigma _{\psi}^2\}$. $\bm{g}(\bm{x})$ is given by:
\begin{equation}
\bm{g}(\bm{x}) = \left[ {\begin{array}{*{20}{c}}
\cos{\psi} & \sin{\psi} & 0 & 0 \\
-\sin{\psi}& \cos{\psi}& 0 & 0 \\
0 & 0& 1 &0 \\
0& 0 &0 &1
\end{array}} \right]
\end{equation}
The simplified system (\ref{simsys}) is a stochastic continuous system with process noise. {Note that even with this simplification, the drone is driven with a full 6-DOF model throughout both training and evaluation phases in simulation, assuming a high-bandwidth controller capable of precisely following these velocity commands.}

\subsubsection{Reward Function}
The reward functions determine the desired behavior of the agents during the training process. The reward function consists of two parts: one is the existential penalty $\rset_P = -w_1/T_{max}$ of the team, where $T_{max}$ is the allowable number of time steps in an episode and $w_1$ is the adjustable weight; the other is the task reward $\rset_T$. The total reward of an agent is $\rset = \rset_P + \rset_T$. The existential penalty serves to accelerate the exploration progress, while the task reward serves to guide the drones in completing their tasks.

For runner agents, there are three objectives: reaching the desired target while evading the attacks from the runner and avoiding colliding with the wall. Hence, the task reward for the runner agent is specified as:
\begin{equation} \label{terminalRewardR}
\setlength{\nulldelimiterspace}{0pt}
{\rset^r_T} = \left\{ \begin{array}{l}
{C_1(1-\|\boldsymbol{p}^{rg}_T\|/\|\boldsymbol{p}^{rg}_0\|}),\quad if\;step\;ends\;at\;T\\
-C_2, \quad if\;agent\;crashes\;
\end{array} \right.   
\end{equation}
where $C_1$ and $C2$ are positive constants, i.e., $C_1, C2 \in \mathbb{R} ^{+}$. $\boldsymbol{p}^{rg}_T$ and $\boldsymbol{p}^{rg}_0$ are the terminal relative position and the initial relative position.
Note that the runner agent crashes with the chasers or wall, receiving the same penalty. The first item is to encourage the runner drone to reach the target. 

For chasers, their objectives are to crash down the runner while avoiding colliding with other teammates and the wall. The task reward for the chaser agent $c_i$ is:
\begin{equation} \label{terminalRewardC}
\setlength{\nulldelimiterspace}{0pt}
{\rset^{c_i}_T} = \left\{ \begin{array}{l}
{C_1},\quad if\;team\;captures\; the\; runner\;\\
{-C_2}, \quad if\;agent\;collides\; with\; the\; wall\; \\
{-C_2\mathcal{H}(d_{\epsilon} - d^i_{min})}, if\;agents\; are\; too\; close\;
\end{array} \right.   
\end{equation}
where $\mathcal{H}(\cdot)$ is a Heaviside step function (unit step function), $d^i_{min}$ is the minimum distance of the chaser $c_i$ with other chasers, i.e., $d^i_{min} = min(\{\|\boldsymbol{p}_{ij}^{cc}\|_2\}_{j\in \mathcal{N} \setminus \{i\}})$ and $d_{\epsilon}$ is the preset risk threshold.

\section{Asynchronous Multi-Stage Deep Reinforcement Learning Approach}
\subsection{AMS-DRL Algorithm}
According to Theorem \ref{existence}, to solve the MPETN problem (\ref{mixed_game}) becomes to find the NE of the adversarial mixed game $\{\mathcal{M}^r, \mathcal{M}^c\}$. Here, we propose the AMS-DRL approach to approximate the NE of the game. AMS-DRL features a cold-start learning stage and an asynchronous learning stage. The AMS-DRL approach is described in Algorithm \ref{alg::target} in detail. 

\subsubsection{Cold-Start Learning Stage}
At the Cold-Start Learning (CSL) stage $S_0$, the objective is to train the runner drone to navigate toward a target. The runner drone will learn to approach the stationary target without colliding with walls. Note that the position of the target is randomly generated during the training process. This stage necessitates an initial navigation policy for the runner drone.

\begin{algorithm}[!tb]
\caption{AMS-DRL}
\label{alg::target}
\textbf{Start the CST stage $S_0$}\\
\KwIn{Initial runner policy $\p^r_{0,\theta_0}$ for the CSL stage.} 
At the start of training: \\
{Initialize training parameters, such as discount factor $\gamma$, learning rate $\alpha_0$, batch size $bs$}\;
\For{episode ~= 0 : maximum episodes}
{
Randomly spawn $\boldsymbol{p}^g$, $\boldsymbol{p}^r$ and $\psi_r$ within the constrained space. \\
\For{step $t=0:T$}
{Execute the policy $\p^r_{0,\theta_t}$ \\
Collect data $\{s_t, a_t, r_t\}$ \\
Calculate the objective function $\mathcal{L}({\p}^r_{\theta_t})$ \\
Update $\theta_t$ via policy gradient as (\ref{sga}) and (\ref{Adv})\\
}
}
\KwOut{Trained runner policy $\p^r_{0,\theta}$}
\textbf{Start the AL stage}\\
\KwIn{Pre-trained runner policy $\p^r_{0,\theta}$, shared initial chaser policy $\p^c_{1,\theta'_0}$, threshold $\eta$} 
\For{Phase $S_i=S_1:S_k$}
{
    \If{$i$ is odd}{Set $\p^c_{i,\theta'}$ as training mode while set $\p^r_{i,\theta}$ as inference mode, i.e., $\p^r_{i,\theta}=\p^r_{i-1,\theta}$}
    \Else{Inverse the mode for $\p^c_{i,\theta'}$ and $\p^r_{i,\theta}$ }
    \For{episode ~= 0 : converged episodes}
{
Randomly spawn the initial global state within the constrained space $s_0 :=\{\boldsymbol{p}^{g}, \{s_i\}_{i \in \mathcal{N} \cup \{r\}}\}$ \\
\For{step $t=0:T$}
{Execute the runner policy $\p^r_{i,\theta_t}$ and the shared chaser policy $\p^c_{i,\theta'_t}$ \\
Collect data $\{s_t, a_t, r_t\}$ from trained agents \\
Calculate the objective function $\mathcal{L}(mode({\p}^r_{\theta_t},{\p}^c_{\theta'_t}))$ \\
Update $mode(\theta_t,\theta'_t)$ via policy gradient as (\ref{sga}) and (\ref{Adv})\\
}
}
Calculate the success rates $sr^r$ and $sr^c$ over the test $ws$ episodes\\
\If{$\|sr^r-sr^c\| \leq \eta$}{Break}
}
\KwOut{The optimal policies, $\p^{r*}_{\theta}$, $\p^{c*}_{\theta'}$}
\end{algorithm}

\subsubsection{Asynchronous Learning Stage}
Once the runner drone has successfully learned the targeted navigation behavior, multiple chaser drones are added to the scene in the Asynchronous Learning (AL) stage to generate multi-pursuit evasion with targeted navigation behavior. The runner drone is placed in inference mode with an initial policy from the CSL stage while the chaser drones start to train. During inference mode, the runner drone moves according to its previously learned behavior and does not improve upon it. The chaser drones are trained asynchronously until they learn to seek out the runner drone. With their basic fundamental behaviors learned, the chaser and runner drones are trained against each other so that they both learn new strategies while also combating the other side’s strategies. 

The training sequence of AMS-DRL is illustrated in Fig. \ref{figureseq}. There are several training phases in the AL stage, denoted by $S_1$ to $S_k$. During each phase, one team is put in training mode while the other is set in inference mode. The operation $mode(\cdot)$ returns the corresponding parameters once the mode of policy is set as training, e.g., $mode(\theta, \theta') = \theta'$ if $\p^c_{\theta'}$ is under training mode. Once the training policy converges, training for one team halts while that for the other team begins. Let $sr^r$ and $sr^c$ denote the success rate (SR) of the runner team and the chaser team over a window size of $ws$, respectively. The whole training stage ends for both teams when the difference in their success rates converges on a preset threshold $\|sr^r-sr^c\| \leq \eta$, which means an approximate NE. Note that it is a successful task when the runner reaches the target safely, while success for the chasers is to capture the runner or push the runner to walls.
\begin{theorem}[Converge Analysis]Assuming that policies $\p^{r}_{i,\theta}$ and $\p^{c}_{i, \theta'}$ converge in every stage/phase $S_i$ with policy improvement. Given any threshold $\eta \in (0,1)$, with finite phases $k$, the output of AMS-DRL converges to the optimal policies $\p^{r*}_{\theta}$, $\p^{c*}_{\theta'}$ with $\|sr^r-sr^c\| \leq \eta$.
\end{theorem}
\begin{proof}
    Let denote the SR of the runner policy $\pi^r_{i,\theta}$ from stage/phase $S_i$ as $sr^r_i \in [0, 1]$, for $i\in\{0,\dots, k\}$ and the SR of the chaser policy $\pi^c_{i,\theta'}$ from $S_i$ as $sr^c_i \in [0, 1]$, for $i\in \{1,\dots, k\}$. We have $sr^c_i+sr^r_i = 1$ for $i\in \{1,\dots, k\}$. In $S_0$, since there is no chaser around, $\pi^r_{0,\theta}$ achieves the highest SR $sr^r_0 = sr^r_{max}$. In phase $S_1$, the objective of the chaser policy $\pi^c_{1,\theta'}$ is {to} bring down the SR of $\pi^r_{1,\theta} = \pi^r_{0,\theta}$. Hence, $sr^r_1 < sr^r_0$ and $sr^r_1 = sr^r_{min}$ since $sr^c_1 = sr^c_{max} = 1- sr^r_1$ due to the Policy Improvement Theorem. For phase $S_2$, we can easily obtain $sr^r_0>sr^r_2>sr^r_1$ since we train the runner policy to achieve better navigation performance. For the following phases $S_i, i={3,\dots,k}$, similarly we have $1 \geq sr^r_0 > sr^r_2 > \dots> sr^r_k > \dots > sr^r_3 > sr^r_1 \geq 0$.
    
    Now, given any threshold $\eta$, we want to find a finite $k$, such that $\|sr^r-sr^c\| \leq \eta$, i.e., $(1-\eta)/2 \leq sr^r_k \leq (1+\eta)/2$. This forms an interval $\left[0.5-\eta/2, 0.5+\eta/2\right]$ for $sr^r_k$. Since we have a finite sequence of $sr^r_k$ such that $1 \geq sr^r_0 > sr^r_2 > \dots> sr^r_k > \dots > sr^r_3 > sr^r_1 \geq 0$, it's guaranteed that there is at least one $k$ such that $sr^r_k$ lies in this interval and policies converge to $\p^{r*}_{\theta}$, $\p^{c*}_{\theta'}$. This completes the proof.
\end{proof}

\begin{figure}[!tbp]
      \centering
      \includegraphics[width=3.2in]{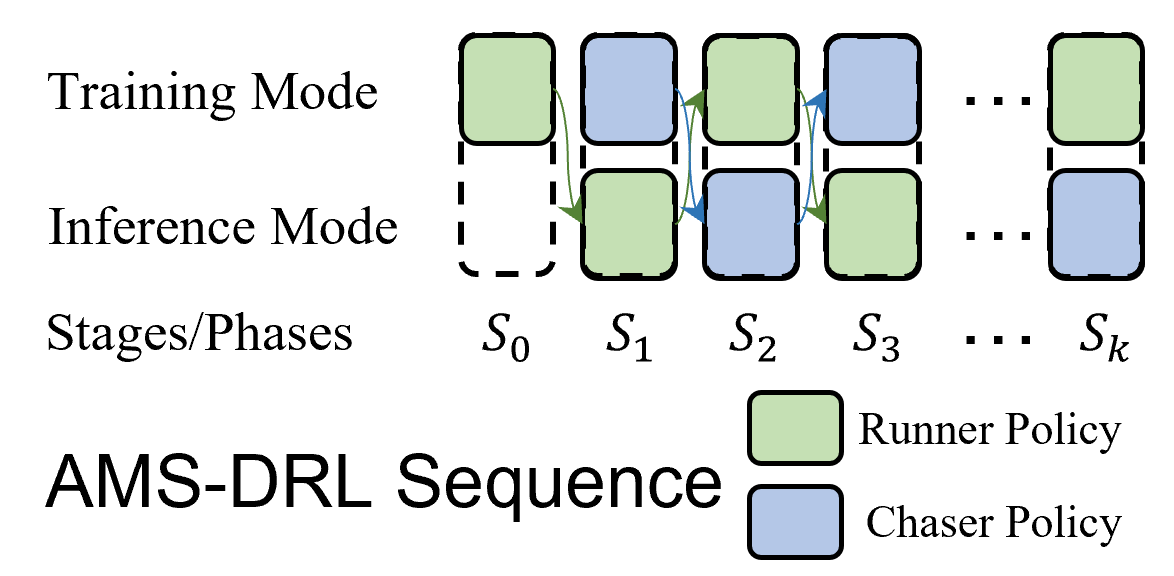}
      \caption{Illustration of the AMS-DRL training sequence for the shared chaser policy and the runner policy.}
      \label{figureseq}
\end{figure}

\begin{figure}[!tbp]
      \centering
      \includegraphics[width=3.3in]{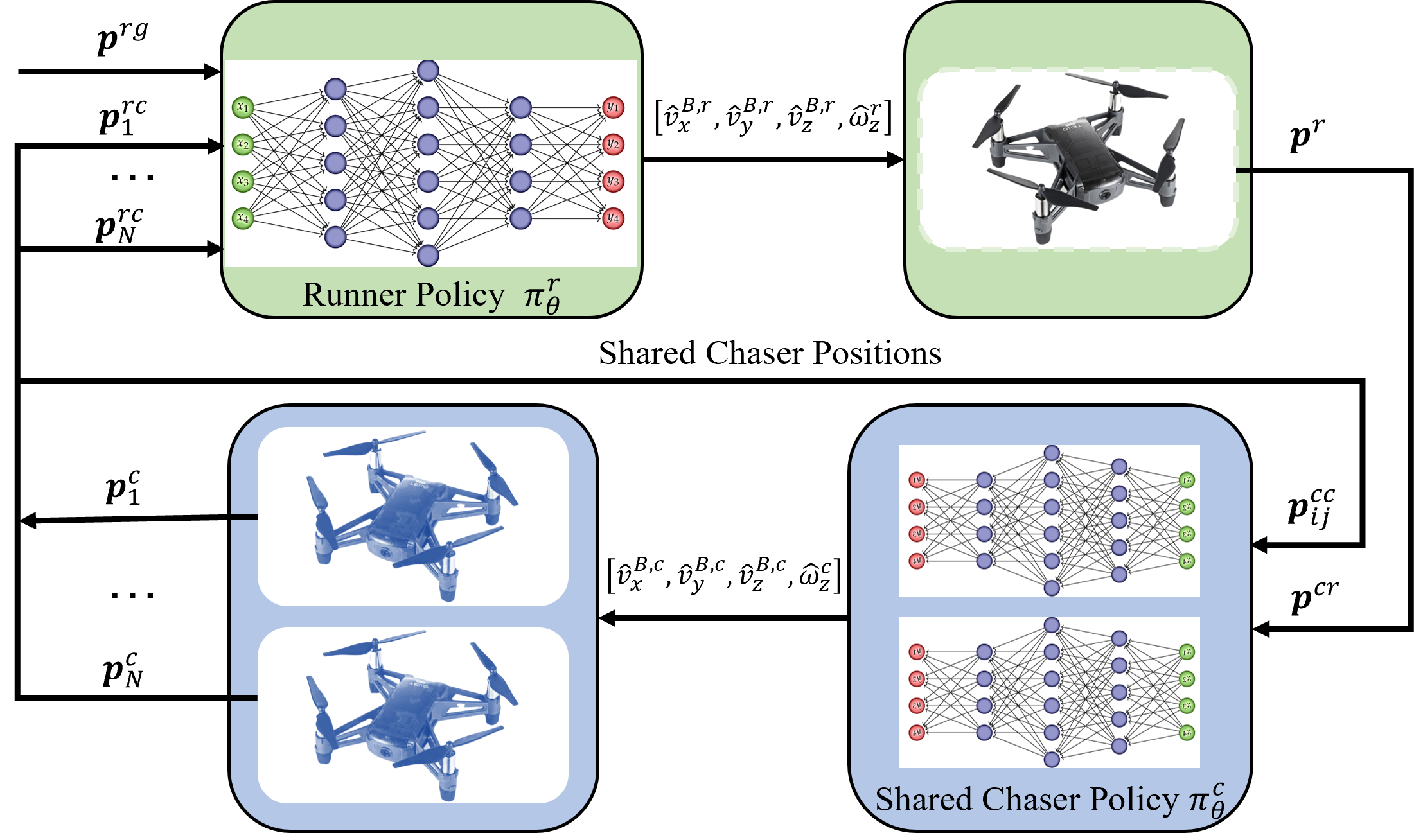}
      \caption{Illustration of the neural network architecture for chaser agents and the runner agent.}
      \label{figure3}
\end{figure}

\subsection{Network Architecture}
As shown in Fig. \ref{figure3}, there are two to-be-trained policy neural networks, namely the runner policy $\p^{r*}_{\theta}$ and the shared chaser policy $\p^{c*}_{\theta'}$. In the actor-critic (AC) framework, the policy network is the actor network, which maps observations to actions directly, i.e., $a^i = \mathcal{F}(o^i\mid \theta),i\in\{r,c\}$. The architectures of the two policy neural networks are similar, with three fully-connected (FC) hidden layers activated by the Sigmoid function. Each layer consists of $512$ nodes, $512$ nodes and $4$ nodes, respectively. The input dimension of the actor network is the same as the dimension of the observation, i.e., $3*(N+1)$ for the runner policy and $3*N$ for the chaser policy. As illustrated in Fig. \ref{figure3}, the runner policy and the shared chaser policy are connected with the state $s_i$ of the agent. The relative positions are obtained from the egocentric observations. Note that, in our scenario, the objectives of these two policy networks are opposite, as described in (\ref{mixed_game}).

\subsection{Training Scheme} In our training process, we utilize the PPO algorithm for the policy update as outlined in (\ref{ppo}). PPO is a robust and effective choice for problems with continuous action spaces, as it employs a ``clipped" surrogate objective and on-policy learning. Hence, the algorithm used in this paper can also be denoted as AMS-PPO. {Note that, in each training stage or phase, convergence is monitored using a specific condition: the training is considered converged if there is no significant improvement in the average reward over the last $l_{sw}$ episodes (1000 episodes in our study) compared to the best average reward achieved so far. Upon convergence, the current stage ends, and the training progresses to the next stage. To prevent premature termination due to initial fluctuations, this convergence criterion is applied only after the first 1 million episodes.} 

\begin{table}[!tb]
\begin{center}
\caption{Basic Specifications of Drones in Simulation}
\label{spec}
\centering
{
\begin{tabular}{@{}l|l@{}}
\toprule
\textbf{Parameter}               & \textbf{Value}     \\ \hline
Number of chasers $N$             &  2              \\
Number of runners & 1 \\
Mass (g) $m$  & 100 \\
Box collider dimension (cm) $L\times W \times H$  & $30\times 30\times 5$  \\
Maximum speed (m/s) $v_{max}$ & 1.0 \\
Maximum angular velocity (rad/s) $\omega_{max}$ & 20.0 \\
\bottomrule
\end{tabular}
}
\end{center}
\end{table}

\section{Experiments and Results}
This section presents the simulation and physical experiments conducted to evaluate the effectiveness of our approach. The simulation experiments cover the training process and inference tests, while the physical experiments focus on the model's ability to transfer from simulation to real-world scenarios and performance in real-time flight tests. Results are analyzed and compared to the traditional Artificial Potential Field (APF) method and other DRL methods. In addition, we investigate the effect of relative maximum speed on the success rate of the navigation task $sr^r$. Furthermore, extensive physical experiments provide a success rate heatmap to examine the impact of the spatial geometry distribution of the runner and chasers on navigation performance.

\subsection{Simulation Experiments}
\subsubsection{Settings}The simulation is run on a workstation. The simulation workstation is used to develop simulation environments for training and testing and export the trained neural network models for deployment. The configuration of the simulation workstation is: AMD Ryzen 9 5900X 12-core CPU, Nvidia RTX 3090 Gaming OC 24GB GPU, and 32GB RAM. Our simulation environment was designed using the Unity rendering engine. The proposed algorithm and training scheme were developed with CSharp and Python based on the ML-Agents Toolkit \cite{juliani2018unity}, which is a flexible training platform for MARL using the Pytorch library. 

The size of the simulation environment in Fig. \ref{figure1} is $5m\times5m\times3m$ (width $\times$ length $\times$ height). One drone is colored white as the runner, while the other two are colored blue as chasers. The target is the red box of dimension $20\times20\times20cm$. The basic specifications of drones are listed in Table \ref{spec}. Note that, in the basic environment, the runner and chasers are the same. The control velocities of drones are limited with $v_{max} = 1m/s$. The goal of the MPETN problem is to avoid physical attacks from the chasers meanwhile to reach the target. The physical attacks mean colliding with the runner of the collider dimension of $30\times30\times5 cm$. The initial linear and angular velocity of drones is reset to zero at the start of each episode. 

The hyperparameters for the AMS-DRL and the PPO training (AMS-PPO) are listed in Table \ref{ams-param}. Note that we set the maximum phases $k_{max}$ as $10$ to avoid inefficient long-time training. The training tools versions are: ml-agents-toolkit: 0.27.0; ml-agents-envs: 0.27.0; communicator API: 1.5.0; PyTorch: 1.8.2+cu11.1. To accelerate the training process, $6$ copies of the environment are placed in one scene, and $3$ concurrent Unity instances are invoked at the start of training. 

\begin{table}[!tb]
\centering
\caption{Parameters of AMS-PPO}
\label{ams-param}
\begin{tabular}{@{}ll|ll@{}}
\toprule
\textbf{Parameter} & \textbf{Value}  &\textbf{Parameter} & \textbf{Value}\\
\hline

Convergence threshold $\eta$ & $10\%$ & Maximum steps $T_{max}$ & 1000 \\
Maximum phases $k_{max}$ & 10 & Test episodes $ws$ & 500 \\
Task reward $C_1$ & 1000 & Task reward $C_2$ & 1000 \\
Existential penalty $w_1$& 10000 & Risk threshold $d_\epsilon$ & $0.5$  \\
\hline

Batch size & 1024 & Buffer size & 10240 \\
Initial learning rate $\alpha_0$ & 0.0003 & beta & 0.01 \\
Clip epsilon $\epsilon$ & 0.2 & lambda $\lambda$ & 0.95 \\
Number of epochs & 3 & Learning rate schedule & linear \\
Checkpoints & 10 & Maximum episodes & 6000000\\
\bottomrule

\end{tabular}
\end{table}

\subsubsection{Baselines}
To evaluate the performance of the proposed approach, we designed several navigation and pursuit strategies (runner/chaser policies) as baselines, including the traditional PID pursuit policy, the APF navigation policy \cite{Falanga2020}, learned navigation policy with PPO \cite{song2021autonomous} {, and the state-of-the-art (SOTA) learning-based pursuit policy CPC-TP in \cite{zhang2022game}.}

(1) Random Pursuit Policy (Random): The chasers move randomly within the enclosed room with a maximum speed of $1m/s$. The initial positions and instant actions of two chasers are randomly generated.

(2) PID Pursuit Policy (PID): The chasers are independently driven by a fast proportional gain controller, aiming to reach the runner's real-time position. The PID controller for the chaser $i$ is formulated as: $a_i^c = [k_p\bm{p}^{cr}_i, 0]$ with $k_p = 2.0$.

{(3) Learned Pursuit Policy (CPC-TP): A pursuit policy trained with CPC-TP in \cite{zhang2022game}, which combines $10$ timestep history observations and an LSTM hidden layer for the policy network and the value network. During the training, the runner moves straight toward a randomly generated target with a maximum velocity $1m/s$. The reward function for the chaser agents is the same as (\ref{terminalRewardC}).}

(4) Manual Navigation Policy (Manual): A manually controlled navigation policy using the keyboard based on the eyes' observation. This is to simulate the manual control mode of drones. Before we evaluate this policy, 50 trials have been practiced to master the keyboard operations. The SR is calculated over 100 episodes.

(5) Learned Navigation Policy (PPO): A one-stage leaning-based navigation policy trained with PPO against the PID pursuit policy. During training, the runner learned to avoid chasers driven by the PID pursuit policy and try to reach the randomly generated target. The training parameters are the same as in Table \ref{ams-param}.

{(6) Learned Navigation Policy (CPC-TP): A navigation policy trained with CPC-TP in \cite{zhang2022game}, against the PID pursuit policy. The policy network and value network are the same as the CPC-TP pursuit policy in (3).}

\begin{figure}[tbp]
      \centering
      \includegraphics[width=3.4in]{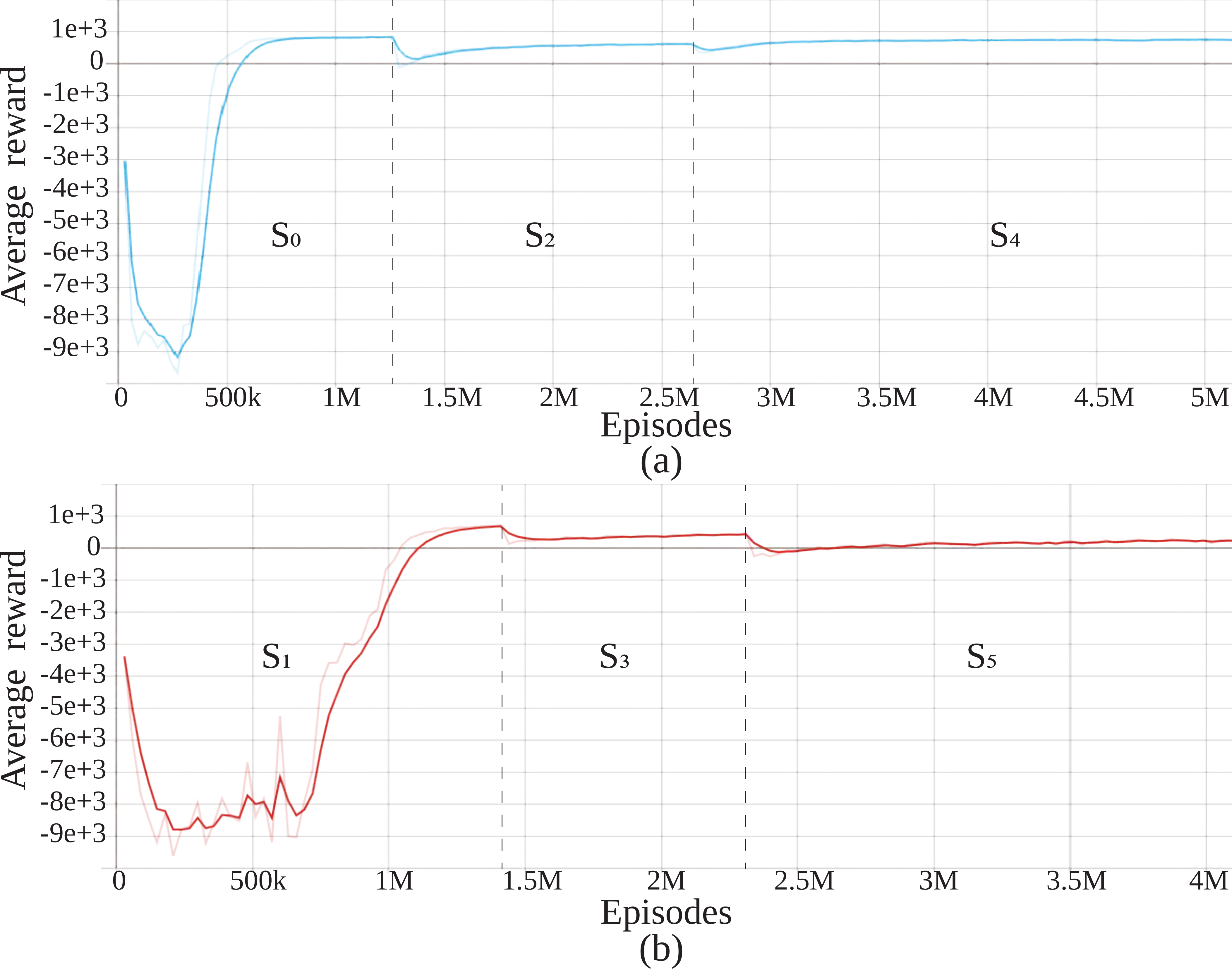}
      \caption{Average reward of AMS-DRL during the training. The training phases are denoted by $S_0$-$S_1$. (a) Runner policy; (b) Chaser policy.}
      \label{reward-train}
\end{figure}

\begin{figure}[tbp]
      \centering
      \includegraphics[width=3.3in]{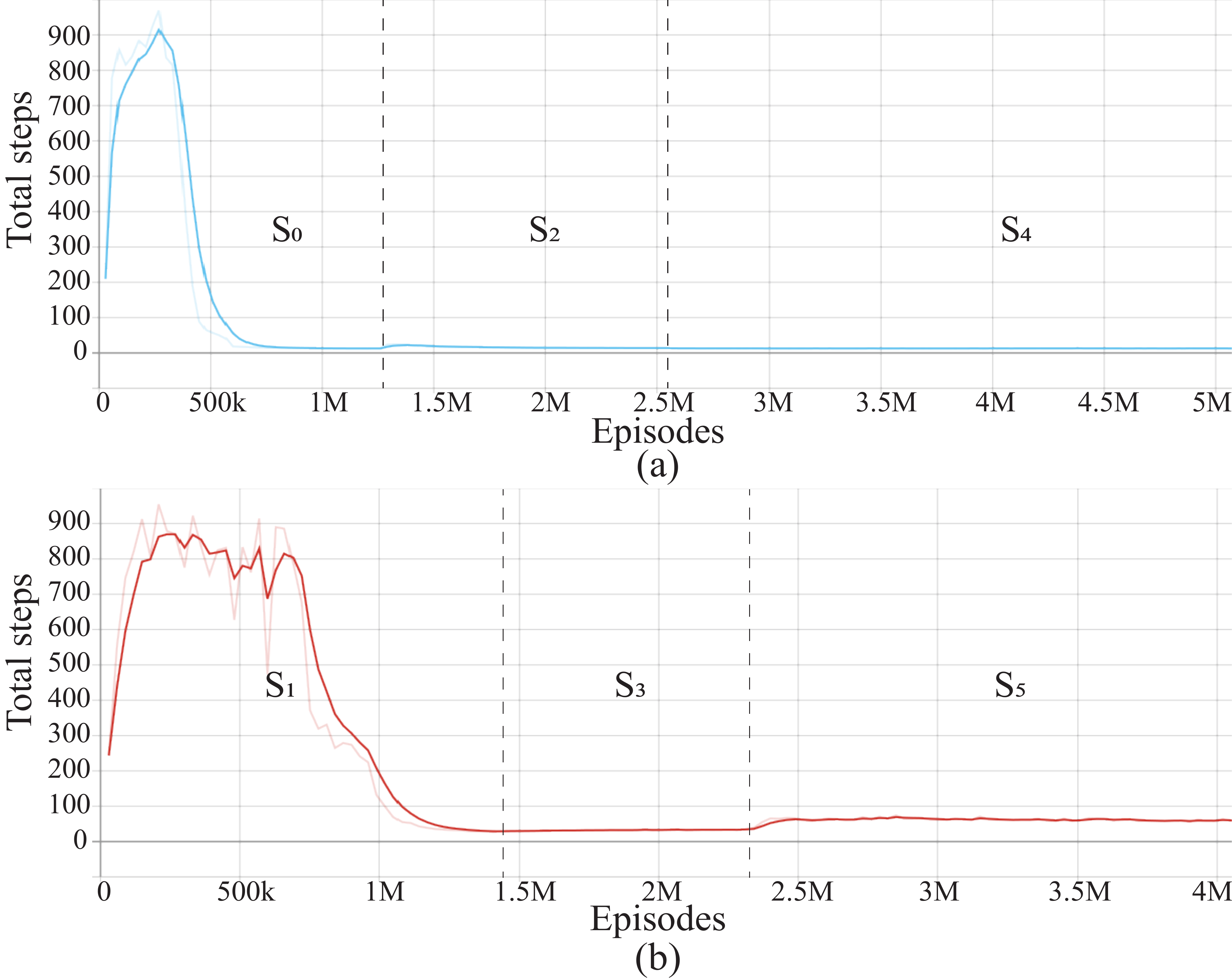}
      \caption{Total steps of AMS-DRL during the training. The training phases are denoted by $S_0$-$S_1$. (a) Runner policy; (b) Chaser policy.}
      \label{steps-train}
\end{figure}

(7) APF Navigation Policy (APF): A traditional heuristic navigation policy with an attractive potential field and a repulsive potential field. In our work, the field functions in \cite{choset2010robotic} are adopted. The attractive potential field is described as:

\begin{equation}
{U_{att}(\bm{p}^r)} = \left\{ \begin{array}{l}
{0.5\xi d^2(\bm{p}^r, \bm{p}^g)},\quad d(\bm{p}^r, \bm{p}^g) \leq d^*_{g}; \\
d^*_{g} \xi d(\bm{p}^r, \bm{p}^g) - 0.5 \xi (d^*_{g})^2, \quad d(\bm{p}^r, \bm{p}^g) > d^*_{g}.
\end{array} \right. 
\end{equation}
where $\xi$ is the attractive force gain, $d(\bm{p}^r, \bm{p}^g)$ is the Euclid distance between $\bm{p}^r$ and $\bm{p}^g$. $d^*_{g}$ is the adjustable distance. The repulsive potential {field} is:
\begin{equation}
{U_{rep}(\bm{p}^r)} = \left\{ \begin{array}{l}
{0.5\zeta (1/D(\bm{p}^r) - 1/D^*)^2},\quad D(\bm{p}^r) \leq D^*; \\
0, \quad D(\bm{p}^r) > D^*.
\end{array} \right. 
\end{equation}
where $\zeta$ is the repulsive force gain and $D(\bm{p}^r)$ is the nearest distance between the runner and the chasers. $D^*$ is the adjustable safe distance. The desired combined force generated from the potential fields {is}:

\begin{equation}
    F(\bm{p}^r) = -\nabla U(\bm{p}^r)
\end{equation}
where $U(\bm{p}^r) = U_{att}(\bm{p}^r)+ U_{rep}(\bm{p}^r)$. $\nabla(\cdot)$ is the gradient vector w.r.t. the position of runner $\bm{p}^r$. Then we can obtain:
\begin{equation}
     \nabla {U_{att}(\bm{p}^r)} = \left\{ \begin{array}{l}
{\xi (\bm{p}^r- \bm{p}^g)},\quad d(\bm{p}^r, \bm{p}^g) \leq d^*_{g}; \\
d^*_{g} \xi (\bm{p}^r - \bm{p}^g)/d(\bm{p}^r, \bm{p}^g), \quad d(\bm{p}^r, \bm{p}^g) > d^*_{g}.
\end{array} \right. 
\end{equation}
\begin{equation}
\nabla {U_{rep}(\bm{p}^r)} = \left\{ \begin{array}{l}
0, \quad D(\bm{p}^r) > D^* \\
{\zeta (1/D^* - 1/D(\bm{p}^r))}\frac{1}{D^2(\bm{p}^r)}\nabla D(\bm{p}^r), others.
\end{array} \right. 
\end{equation}
Here, we set $\xi =2.0$, $\zeta=2.0$, $d^*_g=0.5m$ and $D^*=0.5m$.

\subsubsection{Ablations}
To investigate the effect of asynchronous multi-stage learning, we compared the AMS-DRL with the ``Direct" learning method. The ``Direct" learning method trains two different policies, i.e., runner policy and chaser policy, at the same time with the same configurations of AMS-DRL.

\begin{table}[!tbp]
\begin{center}
\caption{Navigation Performance Comparison (success rate of the runner $sr^r$ \%) against Various Chaser Policies}
\label{tests}
\centering
\resizebox{\columnwidth}{!}{\begin{tabular}{@{}l|cccccc@{}}
\toprule
Chaser Policy              &  \multicolumn{6}{c}{Runner Policy (Navigation Policy)} \\ (Pursuit Policy)& Manual & APF & Direct & PPO  & CPC-TP  & AMS-DRL       \\ \hline
Random & 77.6 & \textbf{99.4} & 21.6 & 94.4& 80.0 & 95.1 \\
PID   &  10.2          &    \textbf{67.1}    &  0.0&    30.3 & 29.5 & 44.1      \\
{CPC-TP} &2.4 & \textbf{77.9}&0.0 & 11.2 & 9.4  & 44.8\\
AMS-DRL      & {\color{red}0.0} & {\color{red}31.8} & {\color{red}0.0} & {\color{red}3.2} & {\color{red}4.6} & {\color{red}\textbf{37.1}} \\
\bottomrule

\end{tabular}
}
\end{center}
\end{table}

\begin{table}[!tb]
\centering
\caption{Comparison of inference runtime in various devices (ms).}
\label{tab:runtime_comparison}
{\begin{tabular}{lccc}
\toprule
Method/Device & Desktop & Xavier NX & RaspberryPi 4B \\
\midrule
CPC-TP\cite{zhang2022game} & 0.152$\pm$0.038 &      0.889$\pm$0.118 &     4.748$\pm$0.892 \\
AMS-DRL(Chaser) & 0.084$\pm$0.011 &      0.580$\pm$0.135 &     2.296$\pm$0.866 \\
AMS-DRL(Runner) & 0.085$\pm$0.006 &      0.536$\pm$0.172 &     1.592$\pm$0.562 \\
\bottomrule
\end{tabular}}
\end{table}
 
\subsubsection{Training Results}
During the training, all objects are randomly generated within the constrained 3D simulation space. The mean rewards and total steps over three epochs of the AMS-DRL training are illustrated in Fig. \ref{reward-train} and Fig. \ref{steps-train}. {We can find that the policies converge at phase $S_5$ with a threshold $\eta = 10\%$. Both runner and chaser policies fluctuate in the early training phases, i.e., $S_0$ and $S_1$ for runner and chaser policies, respectively. In each phase, the trained policy improves until the aforementioned convergence condition is met.} The success rates $sr^r$ and $sr^c$ of $S_5$ are $45.31\%$ and $54.69\%$ respectively. From Fig. \ref{reward-train}, we can observe that at each phase there is a reward drop compared to the last inference policy at the beginning of training. This indicates that their opposites learned to become better from the previous phase.

\begin{figure}[!tbp]
      \centering
      \includegraphics[width=3.4in, trim={0 0.15in 0 0},clip]{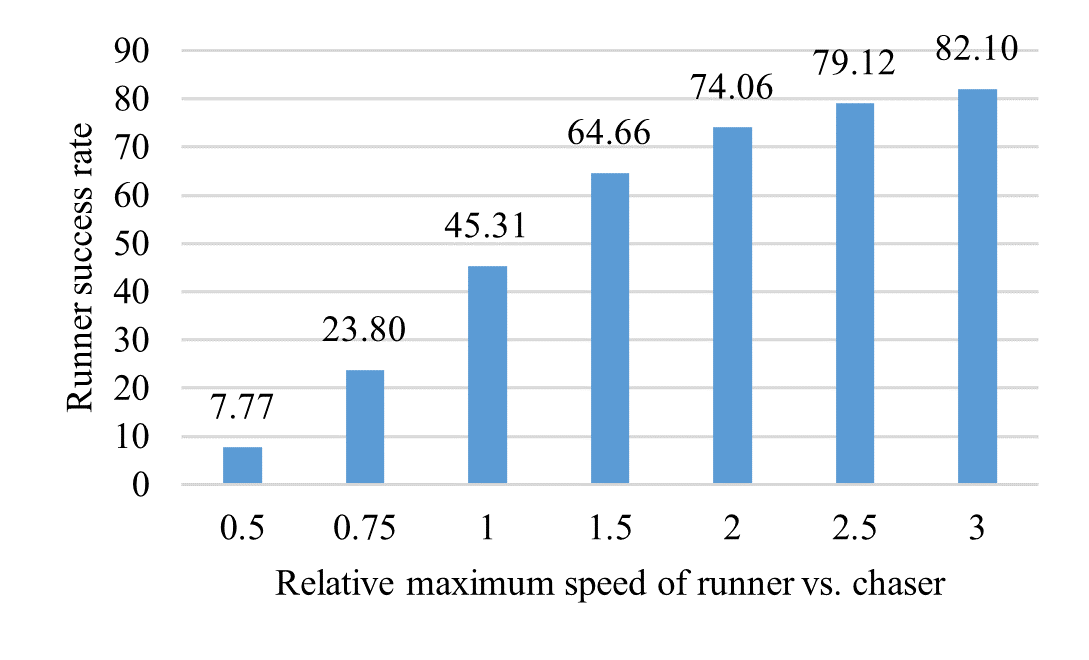}
      \caption{Runner success rate with different relative maximum speed.}
      \label{relative}
\end{figure}

\begin{figure}[!tbp]
      \centering
     { \includegraphics[width=3.4in, trim={0 0.2in 0 0},clip]{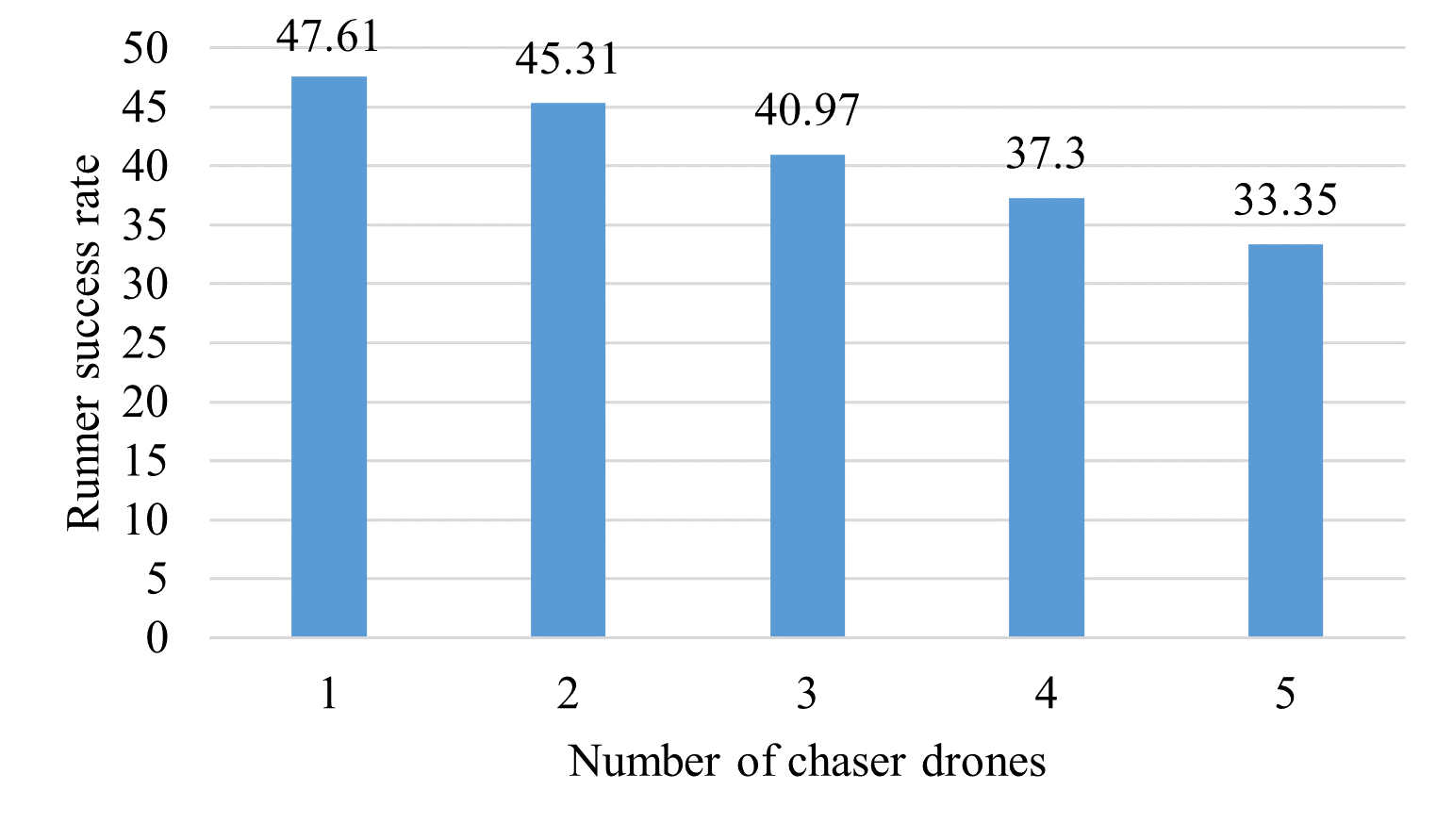}
      \caption{Runner success rate with different numbers of chaser drones.}
      \label{fig:numberChaser}}
\end{figure}

\subsubsection{Testing Results}
In this work, the navigation success rate $sr^r$ (reach the target without collisions) is the evaluation metric for the proposed approach. We tested different combinations of pursuit policies and navigation policies over 500 episodes. The AMS-DRL policies are exported from the last phase $S_5$. The Direct policy is trained from the ablation ``Direct" method. Note that for fair evaluation, during testing, we fixed the target at position $[1.0, 3.5, 0.5]m$ and the runner at position $[1.0, 0.0, 0.2]m$ with random noise inside a sphere of radius $0.2m$. This setting will decrease the navigation performance compared to the training process. The success rates $sr^r$ are listed in Table \ref{tests}. {The inference runtime of different policies is tested on various devices, namely the desktop with Intel i7-10700, Jetson Xavier NX, and Raspberry Pi 4B. From the test results listed in Table \ref{tab:runtime_comparison}, the trained policies can be executed onboard in real time.}

{Compared to other learning-based baselines such as Direct, PPO and CPC-TP}, our proposed learning-based approach AMS-DRL has achieved the best performance against different chaser policies. The direct method cannot learn an effective navigation policy since the runner is unable to focus on reaching the target while avoiding the chasers at the same time. APF obtains the highest success rate against random, PID and CPC-TP pursuit policies while falling behind against the AMS-DRL pursuit policy. {A well-designed APF has been validated with excellent dynamic obstacle avoidance performance in \cite{Falanga2020} when the environment is not cluttered, but our learning-based approach is better at evading collaborative learning-controlled pursuers and with fewer navigation steps (average 182 steps compared to 237 steps for AFP).} Meanwhile, our learning-based approach is more promising when a full nonlinear quadrotor model and individual thrust control are adopted, while APF can only handle the point-mass model. {From column-wise, the chaser policy trained with AMS-DRL outperforms all other chaser policies}, as the success rate decreases when the AMS-DRL chaser policy is adopted. This provides insights into designing pursuit strategies.

\begin{figure}[!tbp]
      \centering
      \includegraphics[width=2.9in]{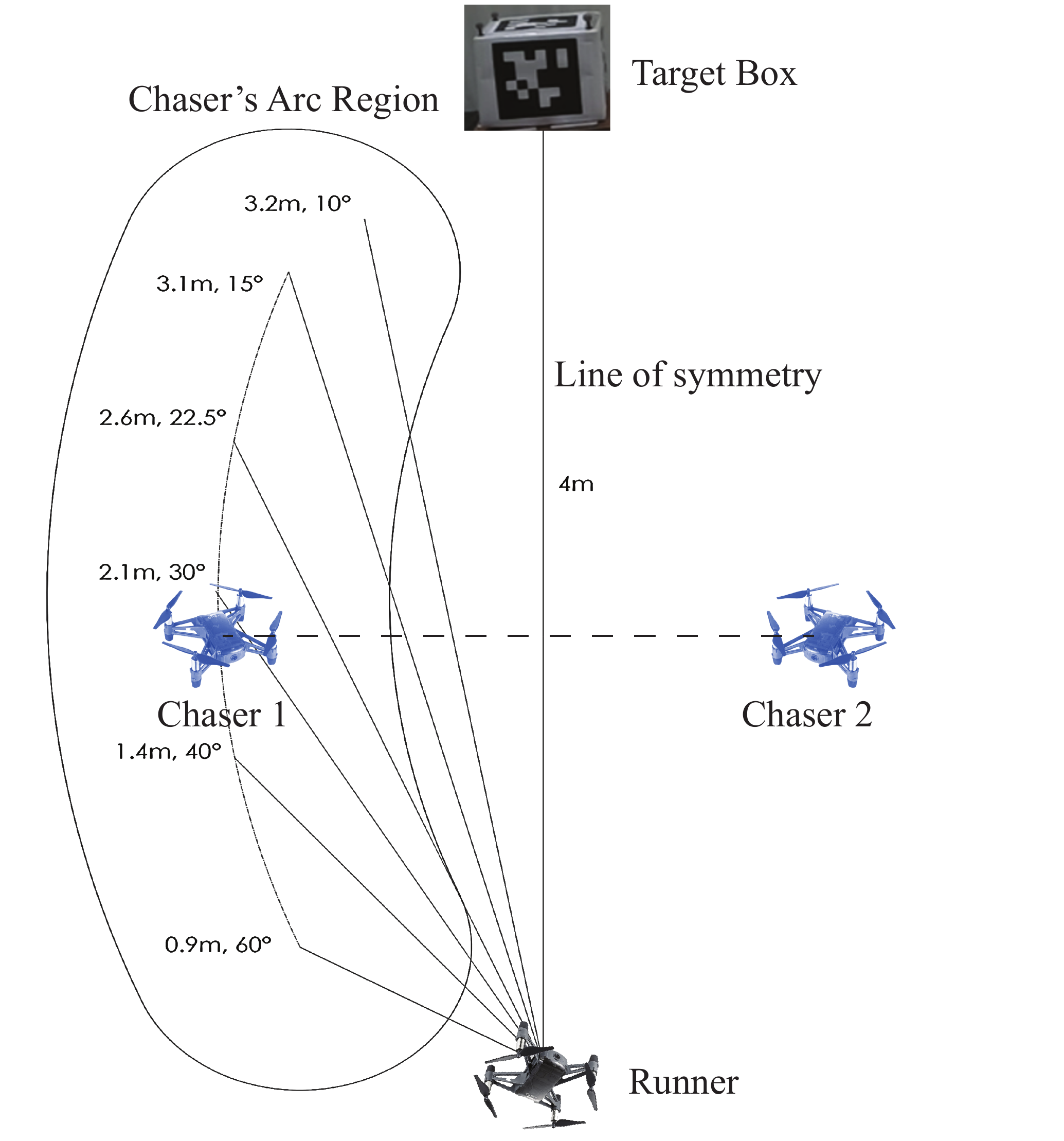}
      \caption{The spacial geometry distribution of agents in the physical tests.}
      \label{distri}
\end{figure}

In the training scenario, the same maximum speed of chasers and runner is assumed, which limits the runner's success rate. To showcase how the relative maximum speed affects $sr^r$, we conducted 1000 simulation tests, each with various relative maximum speeds of runner-chaser from 0.5 to 3. { Note that relative maximum speed $\tilde{v}_{max}$ is defined as the ratio of the maximum speed of the runner to that of the chasers, i.e., $\tilde{v}_{max}=v^r_{max}/v^c_{max}$, thereby providing a comparative measure of their speed capabilities. In our experiments, we fixed the maximum speed of the chasers at $1m/s$ while changing the maximum speed of the runner drone with different relative maximum speeds.} The results in Fig. \ref{relative} show that a faster runner is able to outmaneuver multiple collaborative chasers, especially in cases where the chasers are located between the runner and the target box. {To test our approach in more adversarial scenarios, we trained the policies with different numbers of chaser drones. The results in Fig. \ref{fig:numberChaser} show that more chaser drones could bring down the success rate slightly but the runner policy still performs better than other methods in Table \ref{tests}.}  

\begin{figure*}[!tbp]
\centering
\subfigure[t=0.0s]{\includegraphics[width=2.24in]{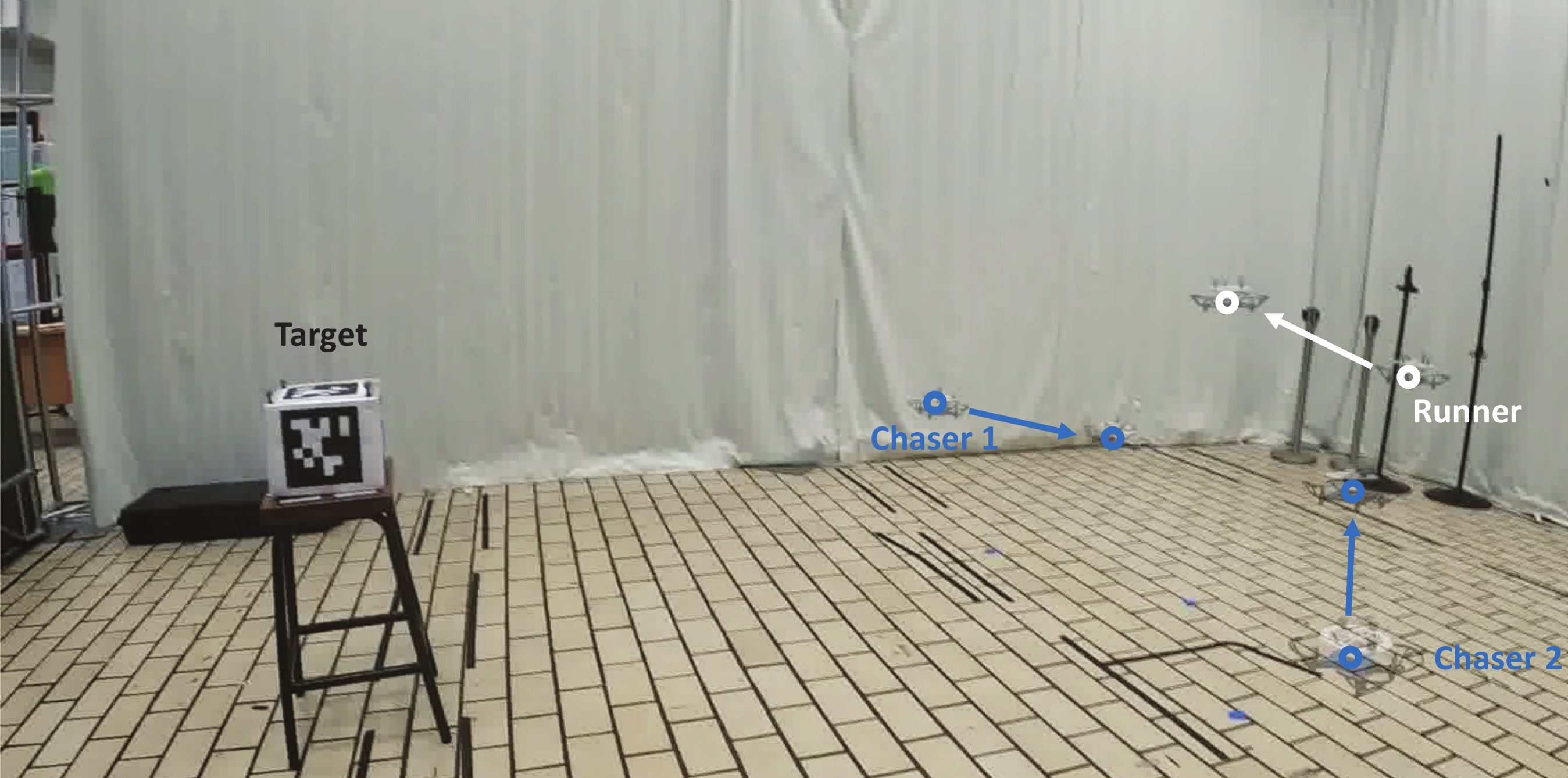}%
\label{fig_first_case}}
\subfigure[t=1.0s]{\includegraphics[width=2.24in]{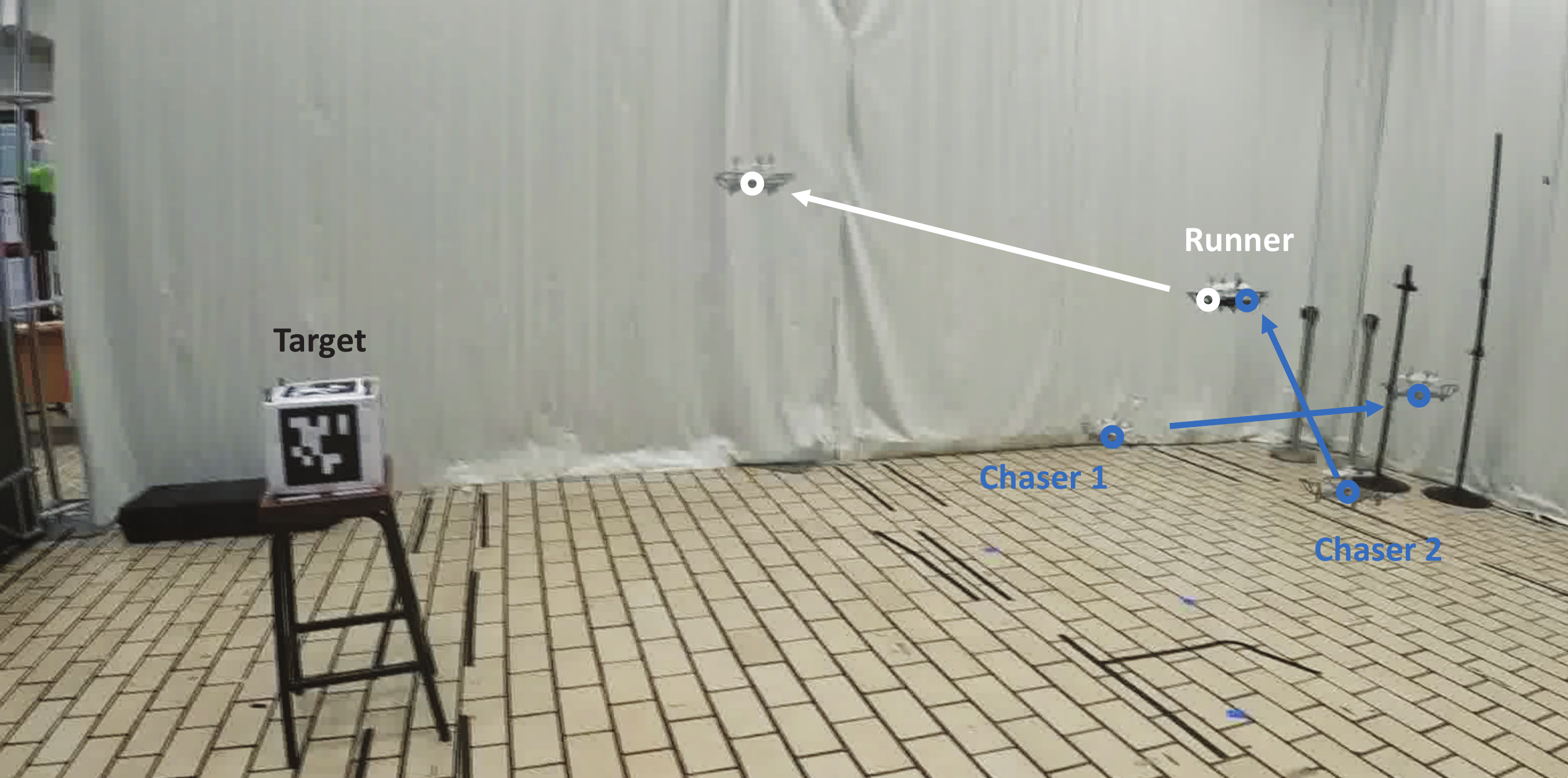}%
\label{fig_second_case}}
\subfigure[t=2.0s]{\includegraphics[width=2.24in]{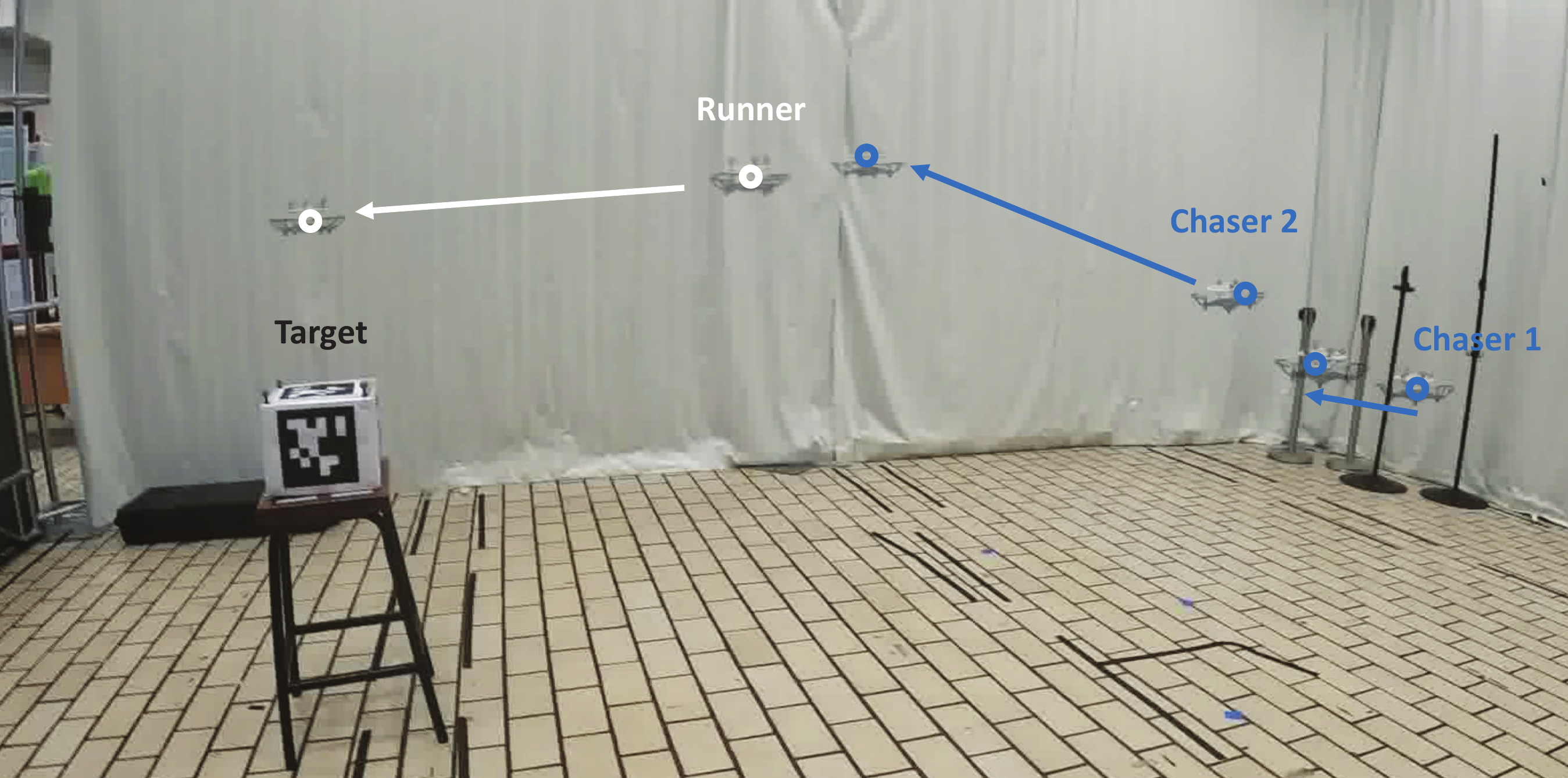}%
\label{fig_third_case}}

\subfigure[t=3.0s]{\includegraphics[width=2.24in]{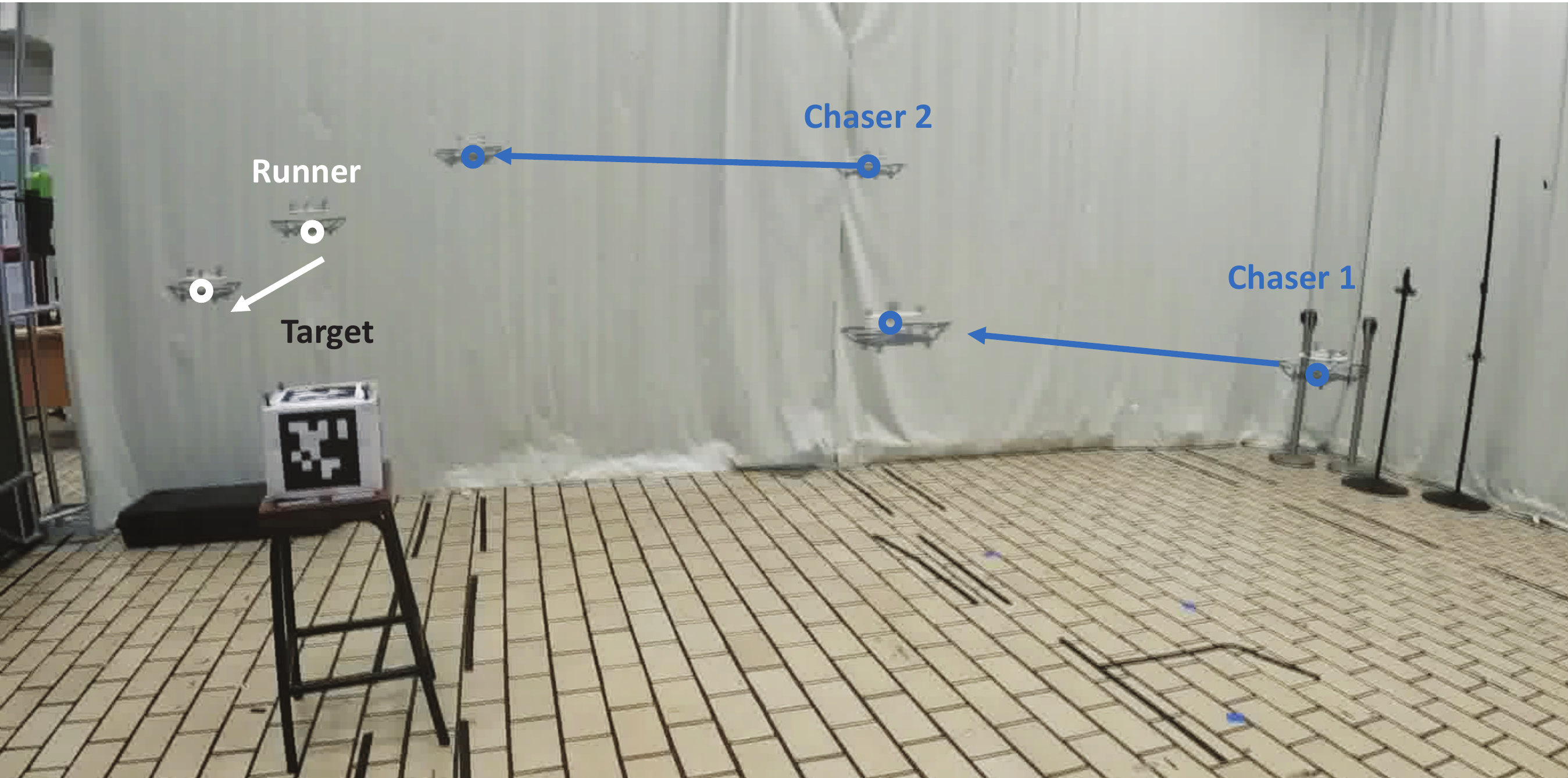}%
\label{fig_four_case}}
\subfigure[t=4.0s]{\includegraphics[width=2.24in]{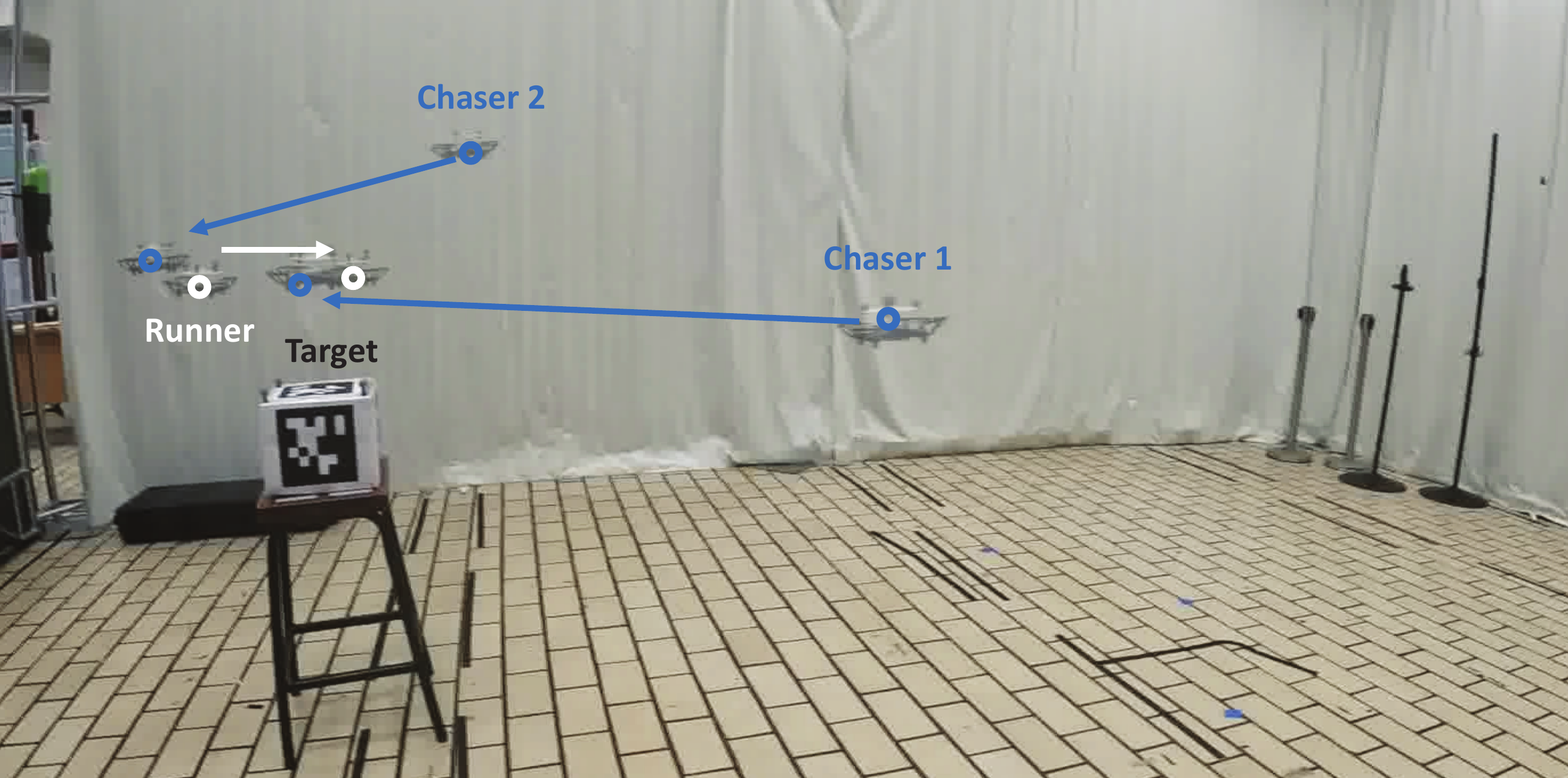}%
\label{fig_five_case}}
\subfigure[t=5.0s]{\includegraphics[width=2.24in]{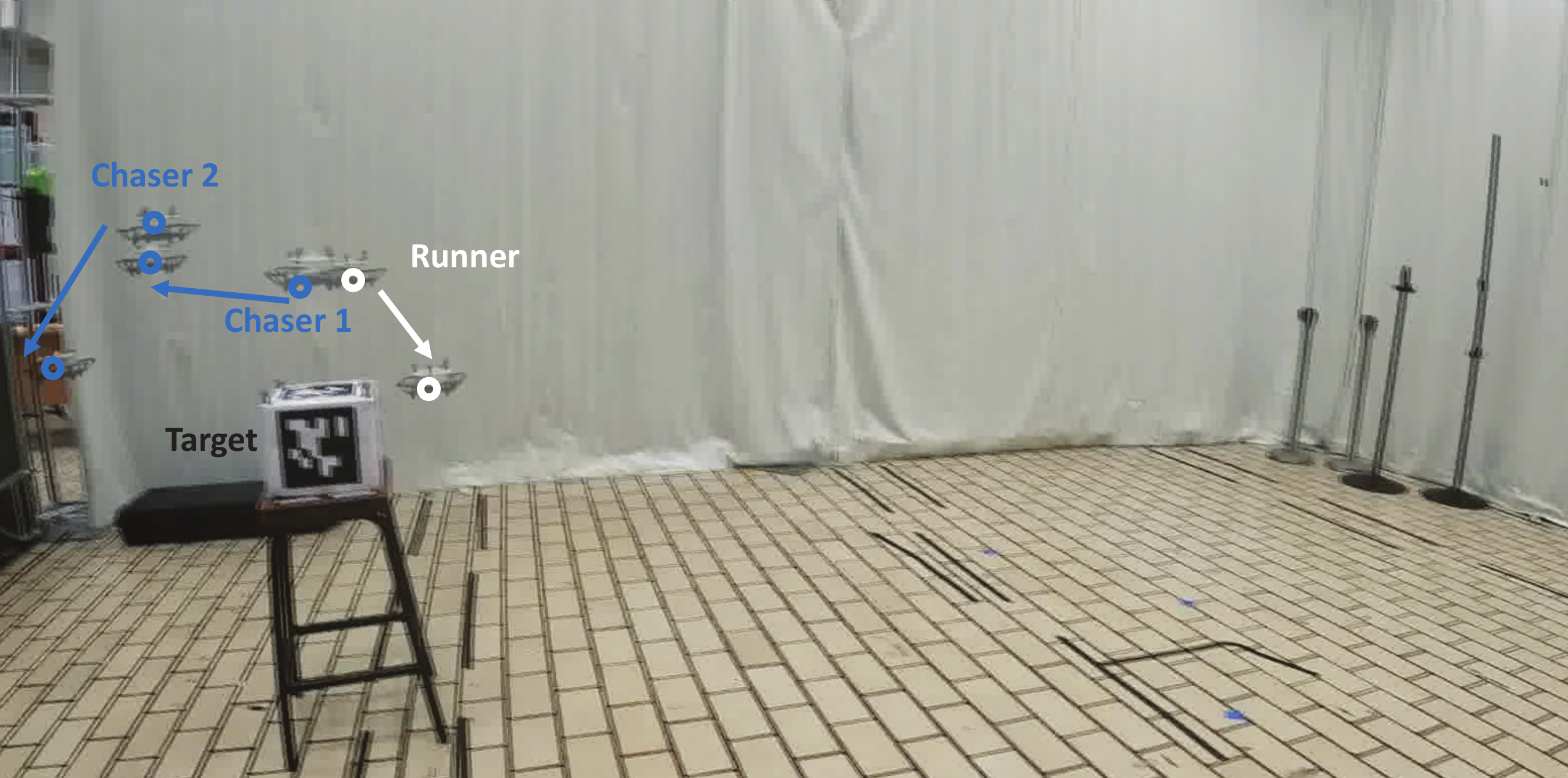}%
\label{fig_last_case}}
\caption{The snapshots of multi-pursuit evasion with targeted navigation in physical experiments. The white runner drone eventually reached the target box without being captured by intelligent blue chasers. The beginning of the arrow is the current position, while the end of arrow is the position in next frame.}
\label{snapshot}
\end{figure*}

\begin{figure}[!tbp]
      \centering
      \includegraphics[width=3.3in]{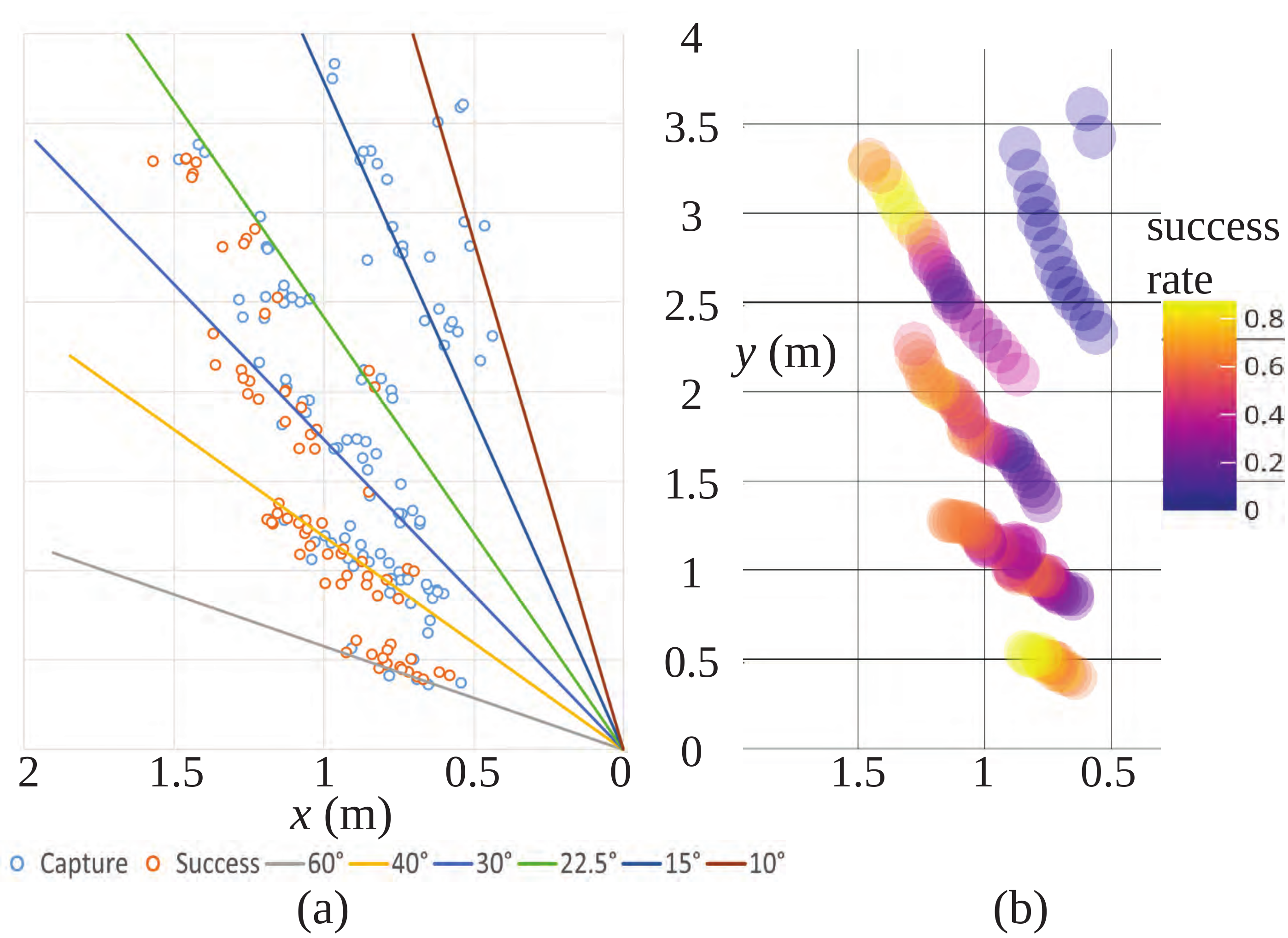}
      \caption{The testing result distribution with different spacial geometry configurations. (a) The cases scatter; (b) the success rate heatmap. The runner is fixed at the origin $[0.0,0.0]$ while the chasers are  mirrored along the y-axis.}
      \label{distri2}
\end{figure}

\subsection{Physical Experiments}
In this section, we present the physical experiments conducted and evaluate the navigation performance of the trained runner policy in real-time flight. To evaluate the models' generalization and Sim2Real capability, we transfer our trained models with the proposed AMS-DRL approach to an unseen large indoor environment to control three Tello Edu drones \footnote{https://www.ryzerobotics.com/tello-edu/downloads} without fine-tuning. Note that both the chaser policy and the runner policy are exported from the final phase of AMS-DRL. More details can be found at \url{https://github.com/NTU-ICG/AMS-DRL-for-Pursuit-Evasion}.
\subsubsection{Settings}
The physical test environment is an indoor flight enclosure with a size of $8 \times 7 \times 4m$. The target is a parcel box with a size of $20 \times 30 \times 20cm$. The runner and two chasers are the same Tello Edu drones of size $20 \times 19 \times 5 cm$, tagged with white and blue colors, respectively. {The detailed specifications of Tello Edu refer to \url{https://www.ryzerobotics.com/tello-edu/specs}.} The real-time position of each drone and the target box was streamed by the OptiTrack motion capture system with tracker balls attached to each object of interest (OOI). Fig. \ref{example} shows one instance of our experiment, where the runner, chasers, and box are OOIs in our experiments.

All drones are driven by ROS Melodic operating on a center computer (Ubuntu 18.04 OS) with a frequency of $240Hz$. The center computer subscribes to the positions of all OOIs and computes the desired command velocities with loaded trained policy models (exported as ONNX files). Each drone is set as AP mode and connected to the existing WiFi network with different IP addresses (192.168.120-192.168.122) during real-time flight tests. The command velocities are sent through the Tello ROS package\footnote{https://github.com/xjp99v5/tello-driver-ros} via a connected WiFi network. During testing, Python3 and ONNXRuntime\footnote{https://onnxruntime.ai/} are used for inference with trained models in ONNX format.

To minimize damage and prevent actual collisions between the chaser and runner drones, a minimum distance of $35cm$ must be maintained between the drones; failure to do so will indicate a corresponding collision outcome. Similarly, a $20cm$ threshold was implemented to detect the arrival of the runner drone at the target box. To fairly evaluate the navigation performance of the runner drone in different situations, the target is fixed in a test position $[-2.0, 0.8, 1.0]m$. The runner drone is placed at a position $[0.45, 1.3, 1.0]m$ with random noise, while the chaser drones’ locations are mirrored along the line of symmetry (LOS) formed between the target box and the runner, as illustrated in Fig. \ref{distri}.

\subsubsection{Sim2Real Performance Evaluation}
In this work, to investigate the navigation performance of the runner policy in various situations, we conducted 189 tests with different spacial geometry configurations. {As illustrated in Fig. \ref{distri}, the chaser's takeoff region (approximated by an arc region) is segregated along 4 lines at 60°, 40°, 30°, and 15° with respect to the LOS, starting at the runner position.} Along each line, several starting positions were denoted with equal spacing between them, and an average of 5-8 flights were conducted at each starting position. Additional lines at 22.5° and 10° were included towards the end of this data collection process for a more comprehensive evaluation. All flight test results are collected and plotted in Fig. \ref{distri2} (a), where the success and capture cases are shown on one side of the corresponding chaser starting region. Here, a success case means the runner drone managed to evade the chasers and reach the target box while a capture case indicates the chaser drones managed to intercept and collide with the runner drone. A successful case is demonstrated in Fig. \ref{snapshot}, which shows the maneuvers performed by the runner and the chasers. From Fig. \ref{snapshot}, the runner drone finally arrived at the target box and meanwhile succeeded in evading the collision from chaser drones. Note that both the drone's body momentum and the air downwash influence its motion, as illustrated in Figs. \ref{snapshot} (b) and (f).

\subsubsection{Results and Discussion}

\begin{table}[!tbp]
\begin{center}
\caption{Navigation Performance of Runner in Real-Time Flight}
\label{success_rate}
\begin{tabular}{@{}lcccccc@{}}
\toprule
Angle      & 60°& 40°& 30°& 22.5° & 15°& 10°                         \\ \midrule
Success       & 17   & 27 & 17 &13  & 0  & 0      \\ 
Capture     & 6   &  31  &   24 & 24 & 20 & 10                                                                                     \\
Total & 23 & 58 & 41 & 37 & 20 & 10 \\
Success rate   & 73.9\%  & 46.6\% & 41.5\% & 35.1\% & 0\% &0\%                   \\
\bottomrule
\end{tabular}
\end{center}
\end{table}

To further analyze the navigation performance of the runner policy, the success rate heatmap is generated over the averaged coordinates of the first 8 closest data points along each angle edge in terms of absolute distance to the chaser drone. This process was repeated iteratively for the next 8 data points until all points along the angle edge were included. The runner drone’s success rate is around $50\%$ along the 40° line, and the success rate decreases when the angle decreases. The success rates with different angles are listed in Table \ref{success_rate}. For angles beyond 90°, the runner drone has a nearly $100\%$ success rate as the chaser drones are behind the runner drone when all of them have the same maximum speed. Note that this simplified physical setup assumes that both chaser drones will spawn mirrored along the LOS, while in a truly random case, the probability is higher since one chaser drone can spawn closer to the runner drone, lowering the runner drone’s success rate.

Among these 189 physical tests, 115 tests were labeled as capture ($60.8\%$) while another 74 tests were labeled as success ($39.2\%$), which benchmarked against the simulation results as listed in Table \ref{tests} ($37.1\%$). The results verified the generalization capabilities of trained policies in a physical application. The differences could be chalked up to the low-level control limitations and the experiment setup. The lack of thrust makes drones unable to maintain the desired flight path under the influence of another drone’s downwash as seen in Fig. \ref{fig_last_case}. An improved navigation policy is expected to be learned with a fully nonlinear model and individual thrust control. The scalability of the proposed approach could be enhanced with the help of an attention mechanism.

\section{Conclusion}

In this paper, we present a novel approach, AMS-DRL, for ensuring the safe navigation of drones under physical attacks from multiple pursuers. The AMS-DRL method trains a deep adversarial neural network to learn from the actions of the pursuers and respond effectively to avoid collisions and reach the target destination. Through extensive simulations and physical experiments, our results show that this approach outperforms existing learning-based baselines in terms of navigation success rate. Additionally, the success rate heatmap provides valuable insights into the impact of spatial geometry on drone navigation performance.

This research highlights the benefits of DRL in solving complex and challenging problems in drone navigation. In addition, the performance of the learning-based chaser policy provides benchmarks for research on anti-drone techniques with DRL. {Further work will leverage individual thrust control and visual sensors to improve the agility and applicability of the runner drone in more adversarial environments.}


%



\section*{Acknowledgment}
The authors would like to thank Thio Teng Kiat with Nanyang Technological University, Singapore, for the help and valuable feedback on this work.

\ifCLASSOPTIONcaptionsoff
  \newpage
\fi



%
\bibliographystyle{IEEEtran}
\bibliography{IEEEabrv,ref}








\end{document}